\documentclass{article}
\usepackage{PRIMEarxiv}
\DeclareMathSizes{6.2}{6.2}{5}{5}
\DeclareMathSizes{7.2}{7.2}{5}{5}

\usepackage{amsmath,amsfonts,bm}

\def\eqref#1{equation~\ref{#1}}

\def\1{\bm{1}}

\DeclareMathAlphabet{\mathsfit}{\encodingdefault}{\sfdefault}{m}{sl}
\SetMathAlphabet{\mathsfit}{bold}{\encodingdefault}{\sfdefault}{bx}{n}

\usepackage{natbib}
\usepackage{enumitem}
\usepackage{hyperref}
\usepackage{url}
\usepackage[T1]{fontenc}
\usepackage{amsmath,amsfonts,amssymb,mathtools,amsthm}
\allowdisplaybreaks[2]
\usepackage{graphicx}
\usepackage{tikz}
\usetikzlibrary{arrows.meta}
\usepackage{microtype}
\newtheorem{theorem}{Theorem}
\usepackage{multirow}
\usepackage{multicol}
\usepackage{siunitx}
\usepackage{tabularx}
\usepackage{subcaption}
\usepackage{algorithm}
\usepackage{algpseudocode}
\usepackage{float}

\newtheorem{proposition}{Proposition}
\hypersetup{
  hidelinks,
  pdftitle={JIVE: Jacobian-Informed Volume Expansion for Diverse Generative Sampling},
  pdfauthor={Guangxun Zhang, Brian Cai, Boxuan Zhang, Chao Chen, and Ruixiang Tang}
}
\setcitestyle{citesep={,}}
\usepackage{booktabs}
\usepackage[table]{xcolor}
\usepackage{makecell}
\usepackage{threeparttable}
\usepackage{enumitem}
\usepackage{placeins}
\graphicspath{{./}{./figures/}}

\title{JIVE: Jacobian-Informed Volume Expansion for Diverse Generative Sampling}

\author{
{\normalfont
Guangxun Zhang$^1$\thanks{Equal contribution.} \quad
Brian Cai$^2$\footnotemark[1] \quad
Boxuan Zhang$^3$ \quad
Chao Chen$^2$\thanks{Equal corresponding authors.} \quad
Ruixiang Tang$^3$\footnotemark[2]} \\[0.6em]
{\normalfont $^1$New York University \quad
$^2$Stony Brook University \quad
$^3$Rutgers University}
}

\definecolor{paperBlue}{HTML}{35618A}
\definecolor{paperOrange}{HTML}{C65D3A}
\definecolor{paperBlueLight}{HTML}{EEF3F7}
\definecolor{paperOrangeLight}{HTML}{FAEFEA}
\colorlet{bestAcol}{paperBlue}
\colorlet{bestBcol}{paperOrange}
\colorlet{bestCcol}{paperBlue}
\colorlet{bestDcol}{paperOrange}

\newcommand{\meanse}[2]{#1{\fontsize{6}{7}\selectfont$\pm$#2}}

\begin{document}

\maketitle
\raggedbottom

\begin{abstract}

Generative models often suffer from mode collapse and limited sample diversity. While prior works attempt to mitigate this by jointly generating a batch of samples and repelling their trajectories, these heuristics do not explicitly maximize the diversity of the resulting endpoints. We introduce JIVE, a training-free framework that enhances generative diversity by injecting velocity perturbations aligned with the leading right singular subspace of the generator's endpoint Jacobian. By leveraging this local geometric structure, JIVE provably maximizes endpoint diversity while preserving sample quality. To maintain practical efficiency, we compute these perturbation directions via matrix-free iterations rooted in classical numerical linear algebra, requiring only a small computational overhead. Across different benchmarks, JIVE boosts both pixel and feature-level diversity in few-step and one-step generation. Code is available at \href{https://github.com/guangxunZhang/JIVE}{\textcolor{blue}{https://github.com/guangxunZhang/JIVE}}.

\end{abstract}

\section{Introduction}
\label{sec:intro}

Diffusion and flow-based generative models are capable of producing high fidelity and prompt-aligned images \citep{ho2020denoising,song2020score,lipman2023flow,liu2023flow}. A large body of effort has been put towards improving the fidelity and speed of such models, yet much less attention is given to the \emph{diversity} of generated outputs. For a given prompt, different noise seeds often produce images with similar composition and semantic features. Low diversity limits the model's usefulness in downstream applications such as data augmentation, content creation, and exploration. This issue is especially problematic in few-step models \citep{sauer2024fast,blackforestlabs2024flux,gandikota2025distillingdiversitycontroldiffusion}, where fewer sampling steps leave less opportunity for trajectories to diverge or for iterative diversity guidance to take effect.

To add diversity into the generation process, a natural solution is to inject a perturbation to the velocity at a sampling point. This has the effect of perturbing the trajectory and thus the endpoint at the target distribution. 
Most prior methods \citep{wu2025oscar,morshed2025diverseflow,corso2024particle} sample multiple seeds at once, and at any time point push these samples away from each other to maximize their difference. We refer to these as batch-repulsion methods. Intuitively, letting samples repulse against each other at intermediate time steps diversifies the trajectories and generation results.

In this paper, we argue that these batch-repulsion methods are only \textbf{implicitly} diversifying the endpoints of the trajectories. To maximize the diversity of endpoints, we should directly find velocity perturbations that maximally change the estimated endpoint. 
To find such velocity perturbations, we propose to interrogate the generator's endpoint Jacobian, which effectively measures how all local perturbations impact the estimated endpoint. 
In particular, we sample our perturbations from the leading right singular subspace of the endpoint Jacobian, which is the subspace spanned by the top-$k$ right singular vectors corresponding to the top-$k$ singular values. \textbf{Our approach finds perturbations that explicitly maximize endpoint diversity.}
See Figure \ref{fig:local-branching}. 

Our method selects perturbations based only on how they affect a single trajectory. This raises the question of how our method could guarantee overall diversity of generated results. We answer this with a theoretical result (Theorem \ref{thm:pairwise-diversity}) that our method increases the pairwise distance between endpoints of multiple generations, and that maximizing local endpoint change also maximizes the increase in spread among generated points.

Algorithmically, we make our method efficient by avoiding explicitly forming the full Jacobian matrix or its singular value decomposition. We instead iterate a random vector towards a top right singular vector with Jacobian and Jacobian-transpose vector products. We also show our method works when power iteration \citep{golub2013matrix}  is applied only a single time instead of requiring full convergence towards a top singular vector. This ensures our method is efficient in actual deployable settings.

\begin{figure}[t]
\centering
\begin{tikzpicture}

\definecolor{figOneInk}{HTML}{35618A}

\node[anchor=south west,inner sep=0] (image) at (0,0)
  {\includegraphics[width=\linewidth]{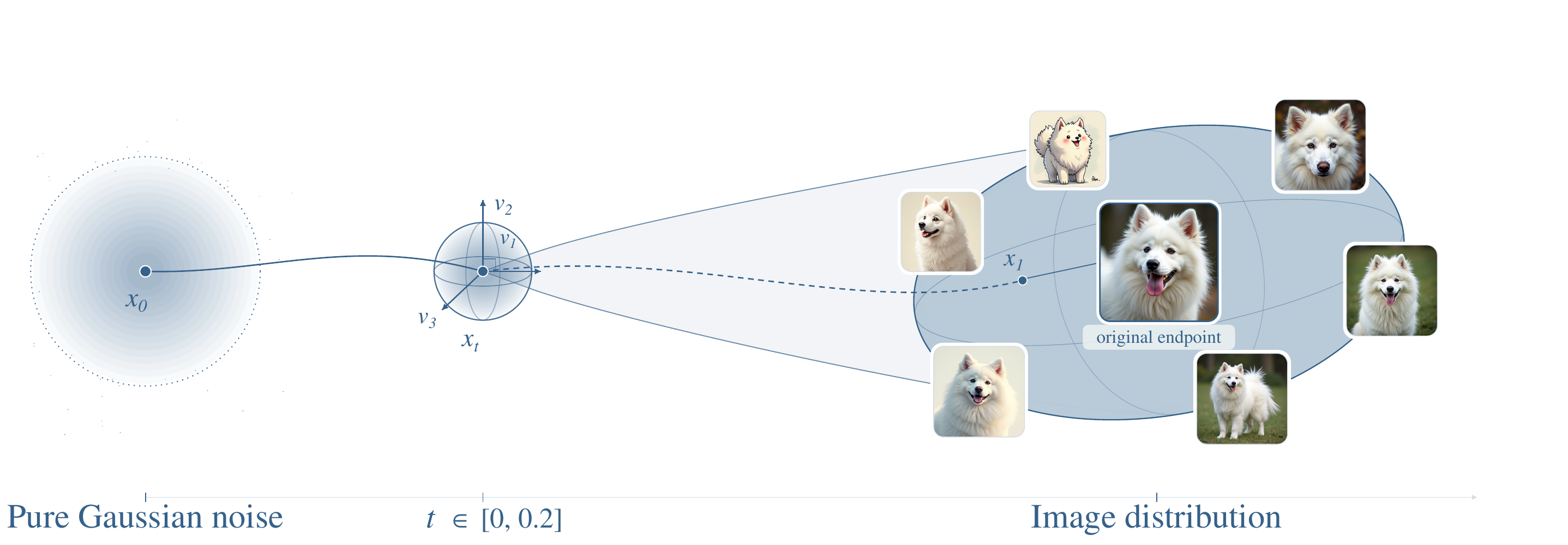}};

\begin{scope}[x={(image.south east)},y={(image.north west)}]

\node[
  text=figOneInk,
  font=\scriptsize,
  align=center
] (patchlabel) at (0.410,0.780)
  {Local perturbation\\space};

\draw[
  figOneInk,
  line width=0.45pt,
  -{Latex[length=1.1mm,width=0.8mm]}
]
  ([yshift=-1pt]patchlabel.south)
  -- (0.338,0.520);

\node[
  text=figOneInk,
  font=\scriptsize,
  align=center
] (volumelabel) at (0.78,0.88)
  {Maximal endpoint volume};

\draw[
  figOneInk,
  line width=0.45pt,
  -{Latex[length=1.1mm,width=0.8mm]}
]
  ([yshift=-1pt]volumelabel.south)
  -- (0.764,0.770);

\end{scope}

\end{tikzpicture}

\caption{
At an intermediate 
trajectory point $x_t$, JIVE injects a local velocity perturbation sampled from the \textbf{leading right 
singular subspace of the endpoint Jacobian}. This explicitly 
maximizes \textbf{endpoint diversity}, as theoretically 
guaranteed by Theorem \ref{thm:pairwise-diversity}. 
}
\label{fig:local-branching}

\end{figure}
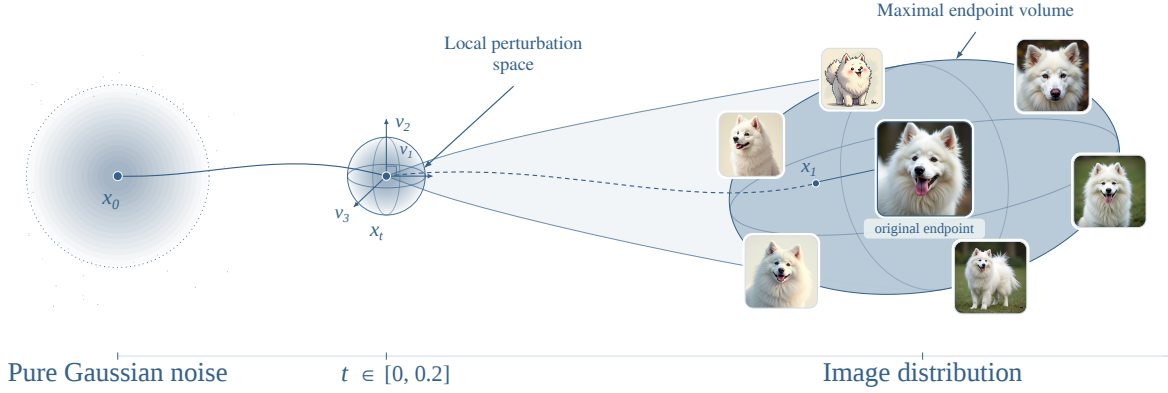

\begin{figure}[bt!]
    \centering
    \includegraphics[width=\linewidth]{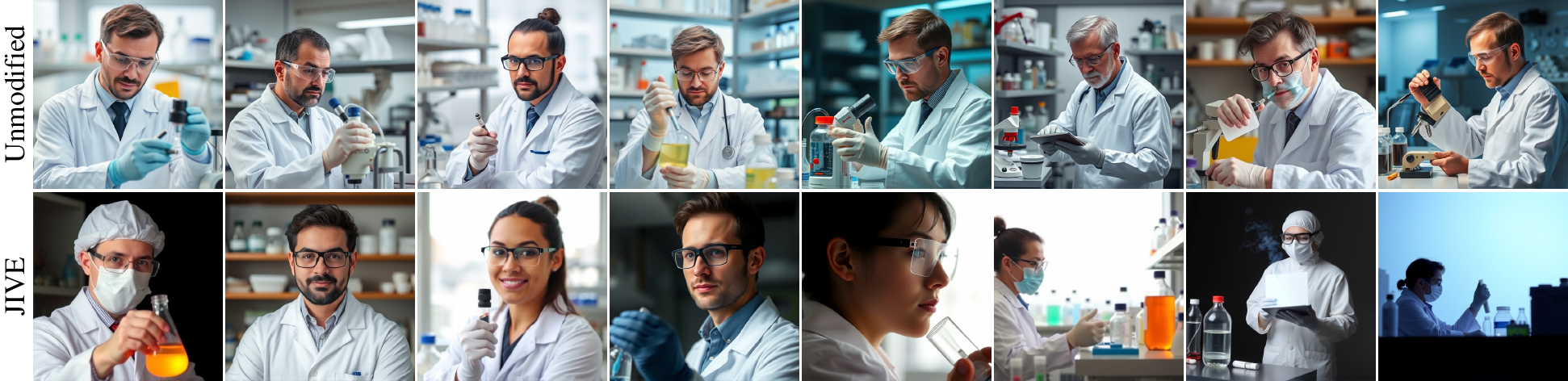}
    \caption{\textbf{Qualitative comparison from matched initial latents.}
    Visual comparison of eight samples generated with 4-step FLUX.1-schnell using the same initial seeds for the prompt ``\textit{\textbf{A scientist}}''.
    The samples of the original method (unmodified) have similar compositions, whereas JIVE produces  distinct subjects and backgrounds.}
    \label{fig:comparison}
\end{figure}


In summary, our paper consists of the following contributions.
\begin{itemize}[topsep=2pt, itemsep=2pt, parsep=0pt, partopsep=0pt]
\item We propose a novel method to inject diversity into modern flow matching generative models. Unlike previous methods, we explicitly perturb the velocity to maximize the endpoint change by sampling directions from the top singular value space of the endpoint Jacobian.

\item We provide a theoretical guarantee for the diversity improvement of the endpoints. 

\item We provide efficient algorithms to compute the velocity perturbation even though the Jacobian itself can have millions of entries. 
\end{itemize}

Across different standard benchmarks \citep{yu2022scaling, ghosh2023geneval}, JIVE consistently enhances generation diversity across multiple metrics, including Feature Vendi and Pixel Vendi. It demonstrates superior or on par performance compared with baseline flow matching models and existing state-of-the-art methods.
Figure \ref{fig:comparison} illustrates the diversity introduced by our method while maintaining the original semantics from the text prompt.


\begin{figure}[!b]
\centering
\includegraphics[width=0.9\textwidth]{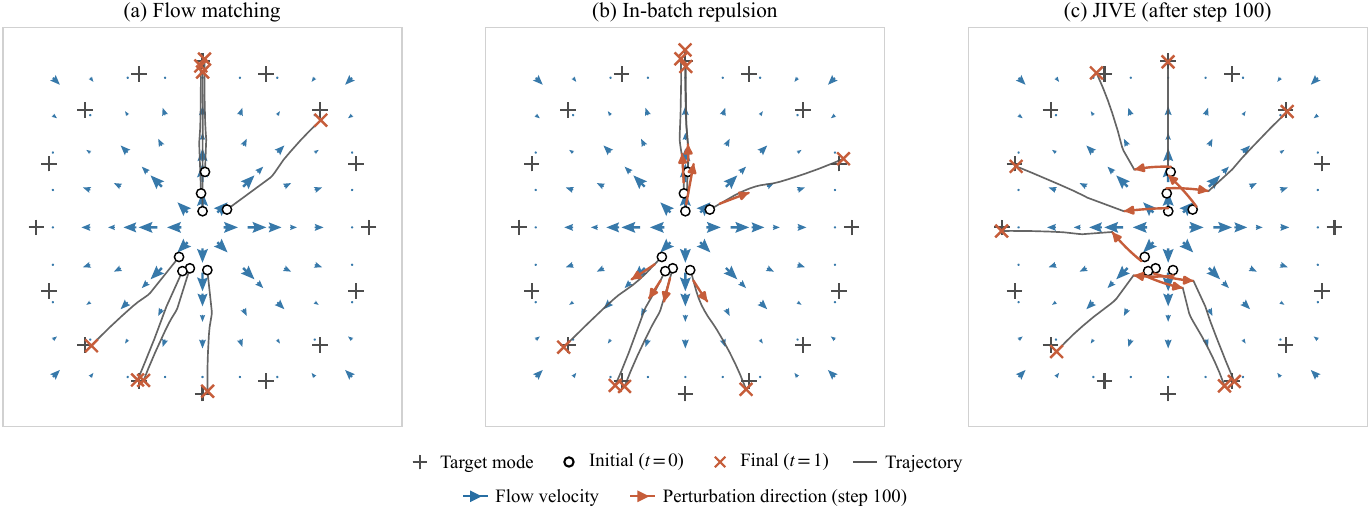}
\caption{ In a toy example, we compare unmodified flow matching (Fig.(a)), a batch repulsion method (OSCAR) (Fig.(b)) and JIVE (Fig.(c)). Eight samples flow from the center toward sixteen Gaussian modes arranged on a ring. Blue arrows are the velocity field, orange arrows are velocity perturbations, and gray curves are the trajectories of the samples.
In panel (a), many points flow to the same few modes. In panel (b), batch repulsion pushes trajectories further apart, but the repulsion direction does not maximize endpoint spread. In panel (c), JIVE finds perturbations maximizing endpoint diversity and pushes points towards new modes. See Appendix~\ref{app:repulsion-limits} for details.}
\label{fig:oscar_steps}
\end{figure}

\section{Method}
\label{sec:method}

\noindent \textbf{Preliminaries.}
Flow matching learns a velocity field $u_\theta(x,t)$ that transports Gaussian noise to data \citep{lipman2023flow,liu2023flow}. To train the model, we interpolate a noise
sample $x_0$ and a data sample $x_1$ as
$x_t=(1-t)x_0+tx_1$, and minimize the loss
$\mathbb E_{x_0,x_1,t}\|u_\theta(x_t,t)-(x_1-x_0)\|_2^2$. During sampling, we start from $x_0\sim\mathcal N(0,I_d)$ and integrate the velocity field $\mathrm dx_t=u_\theta(x_t,t)\,\mathrm dt$ from noise at $t=0$ to data at $t=1$. Although sampling is iterative, we can find a quick approximation of a trajectory's endpoint using the \emph{one-step predictor}
$F_t(x)=x+(1-t)u_\theta(x,t)$. 
We use this predictor to analyze local generator behavior in the following sections.

\subsection{Motivation: independent direction selection}
\label{sec:motivation}
\label{subsec:few-step-motivation}
JIVE finds perturbation directions that change a trajectory's endpoint the most.
Prior methods mostly achieve this goal by sampling a batch of seeds together and pushing their trajectories apart during sampling \citep{wu2025oscar,morshed2025diverseflow}. For example, OSCAR \citep{wu2025oscar} maximizes the volume spanned by trajectory endpoints. We demonstrate with our toy model in Figure \ref{fig:oscar_steps} how this strategy might fail to maximally spread endpoints.

The toy example in Figure \ref{fig:oscar_steps} consists of sixteen Gaussian modes whose centers are placed on the boundary of a ring. These modes serve as the target distribution of a flow matching model. We sample initial points from a Gaussian in the center of the plane. We equip the example with the velocity field that minimizes the flow matching loss. In this example, the field points toward the closest Gaussian mode at any location in this plane. We initialize eight starting points near the center and have them travel along the velocity field. Using the original flow matching velocity field, many starting points end up in the same mode due to their similar initial positions (Figure \ref{fig:oscar_steps}(a)), showing low diversity. 

To perturb points towards new modes we can push them using directions that rotate the point along the ring. This changes the point's closest Gaussian mode and redirects the trajectory towards a different endpoint. For this example, OSCAR’s volume maximizing objective (Figure \ref{fig:oscar_steps}(b)) generates perturbation directions (red arrows) that point directly towards the outer circle, and remain the same across all update steps. This is because the volume gradient mainly points outwards from the mean of the batch of points. 

We argue that useful perturbation directions can be derived using local geometry. The endpoint predictor's Jacobian provides a tool for analyzing local geometry since it tells how local input directions are transformed into changes in the output.
For example, in Figure \ref{fig:oscar_steps} the rotational direction is amplified by the Jacobian matrix. By multiplying a random vector with the Jacobian, the rotational component of the random vector survives while other components become discarded. JIVE uses this idea to select directions independently for each
sample. Figure \ref{fig:oscar_steps}(c) shows how JIVE is able to place most points into new modes and how perturbations are along rotational components. The next section formalizes our direction selection criterion
and describes how JIVE estimates the perturbation subspace.

\subsection{Algorithm}
\label{subsec:algo}
We seek velocity perturbations that produce the greatest change in the endpoint of a sampling trajectory. We therefore use the \emph{leading right singular subspace} of the endpoint predictor Jacobian $J_t=\nabla F_t(x_t)$, i.e., the subspace spanned by the right singular vectors with the $k$ largest singular values. Perturbation vectors sampled from such a subspace cause the \emph{greatest endpoint change} to the generator. Formally, let $Q\in\mathbb R^{d\times k}$ be an
orthonormal basis for the leading right singular subspace at a trajectory point $x_t$. We sample a perturbation
$\delta$ uniformly from the sphere of radius $\varepsilon$ in
$\operatorname{span}(Q)$, and resume sampling from $x_t+\delta$. In practice, we only need to perturb the trajectory once.

To pick an appropriate norm of the velocity perturbation, we use the heuristic that the distribution of $x_t+\delta$ should not be too different from the distribution of $x_t$. We choose a bound for the norm of $\delta$. For $x_t\sim\mathcal{N}(0,I_d)$, the norm of $x_t$ is almost always near $\sqrt{d}$. The standard deviation of the norm is about $1/\sqrt{2}$. We want the norm of $x_t+\delta$ to stay near $\sqrt{d}$. We found that $\delta$ and $x_t$ are usually nearly orthogonal to each other, so we can write $\|x_t+\delta\|_2^2=\|x_t\|_2^2+\|\delta\|_2^2$. 
We try to not have $\|x_t+\delta\|_2=\sqrt{\|x_t\|_2^2+\|\delta\|_2^2}$ be too much greater than $\sqrt{d}+c/\sqrt{2}$ where $c$ is some constant representing the number of standard deviations by which we allow the norm to deviate. On Flux.1-schnell which has latent dimension $64\times64\times16$ at $512\times512$ resolution and $c=2$, the bound for the norm with our heuristic is $26$.

\subsection{Theoretical guarantee of gained diversity}
\label{subsec:pairwise-diversity}
The preceding analysis selects directions using only the local endpoint change of individual trajectories.
We now show how local perturbations affect global spread among many independently generated outputs, measured by expected pairwise squared endpoint distance. Perturbations with zero mean conditional on the latent add covariance to each linearized endpoint, yielding the following connection to our local objective.
\begin{theorem}[Perturbations increase expected pairwise distance]
\label{thm:pairwise-diversity}
Let $X$ be a random latent, $J(X)=\nabla F_t(X)$, and $Q(X)$ have $k$ orthonormal columns. Conditional on $X$, draw $Z$ uniformly on the unit sphere in $\mathbb R^k$ and set $\delta=\varepsilon Q(X)Z$ so that $\|\delta\|_2=\varepsilon$. Define $Y=F_t(X)$ and $\widetilde Y=Y+J(X)\delta$. 
The perturbed endpoints have expected squared pairwise distance 
\[\mathbb E_{X_1,X_2}
    \|\widetilde Y_1-\widetilde Y_2\|_2^2=\mathbb E_{X_1,X_2}
    \|Y_1-Y_2\|_2^2+\xi_{\epsilon,k}, 
    where 
\]
\begin{equation}
\label{eq:increase}
\xi_{\epsilon,k}=
\frac{2\varepsilon^2}{k}\,\mathbb E_{X}
\operatorname{Tr}\!\left[
Q(X)^\top J(X)^\top J(X)Q(X)
\right].
\end{equation}
This theorem means that the proposed velocity perturbation will increase the expected spread of linearized endpoints by $\xi_{\epsilon,k}$. Note the gap $\xi_{\epsilon,k}$ is controlled by the perturbation magnitude $\epsilon$ and increases quadratically with it. 

\end{theorem}
Figure~\ref{fig:pairwise-diversity} depicts a concrete picture of how perturbations can increase pairwise distances. The trace term of Theorem \ref{thm:pairwise-diversity} connects directly to our endpoint displacement objective. Consider the value $\|F_t(x_t+\delta)-F_t(x_t)\|_2^2$, which captures how far a perturbation $\delta$ causes the endpoint to change. If we uniformly draw perturbations $\delta$ with norm $\varepsilon$ from a $k$ dimensional subspace $Q$, under a linear approximation this term becomes
\begin{equation}
\label{eq:displacement}
\mathbb{E}_{\delta}\|F_t(x_t+\delta)-F_t(x_t)\|_2^2\approx \mathbb{E}_{\delta}\|J_t\delta\|_2^2=\frac{\varepsilon^2}{k}\operatorname{Tr}(Q^\top J_t^\top J_tQ)
\end{equation}
The same trace term appears in the equation as in Theorem \ref{thm:pairwise-diversity}. Due to the expectation term of Equation~(\ref{eq:increase}), optimizing Equation (\ref{eq:displacement}) per individual $x_t$ also optimizes overall spread increase. The top right singular subspace maximizes Equation (\ref{eq:displacement}) due to its trace-maximization property.


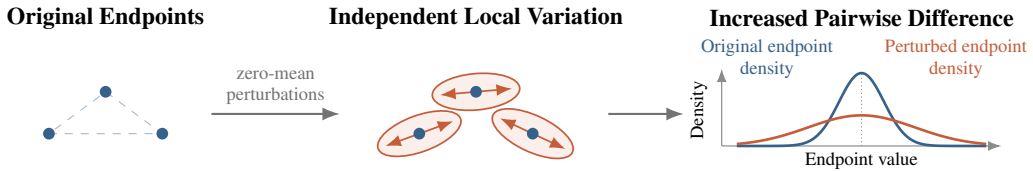
\begin{figure}[!htpb]
\centering
\begin{tikzpicture}[x=1cm,y=0.65cm,>=Latex,font=\footnotesize]

\node[font=\footnotesize\bfseries] at (1.4,1.95)
  {Original Endpoints};
\node[font=\footnotesize\bfseries] at (6.3,1.95)
  {Independent Local Variation};
\node[font=\footnotesize\bfseries] at (11.4,1.95)
  {Increased Pairwise Difference};

\draw[paperBlue!45,dashed]
  (0.65,-0.4)--(2.15,-0.4)--(1.4,0.45)--cycle;

\foreach \xx/\yy in {
  0.65/-0.4,
  2.15/-0.4,
  1.4/0.45
}{
  \fill[paperBlue] (\xx,\yy) circle[radius=2.2pt];
}

\draw[->,thick,black!55] (2.8,0)--(4.5,0)
  node[midway,above,align=center,font=\scriptsize]
  {zero-mean\\perturbations};

\foreach \xx/\yy/\ang in {
  5.55/-0.4/20,
  7.05/-0.4/-25,
  6.3/0.45/5
}{
  \begin{scope}[shift={(\xx,\yy)},rotate=\ang]
    \draw[
      paperOrange,
      fill=paperOrangeLight,
      line width=0.65pt
    ]
      (0,0) ellipse[
        x radius=0.58cm,
        y radius=0.24cm
      ];

    \draw[<->,paperOrange,line width=0.65pt]
      (-0.48,0)--(0.48,0);
  \end{scope}

  \fill[paperBlue] (\xx,\yy) circle[radius=2.2pt];
}

\draw[->,thick,black!55]
  (8.05,0)--(9.05,0);

\draw[->,black!45,line width=0.45pt]
  (9.55,-0.65)--(13.25,-0.65);

\draw[->,black!45,line width=0.45pt]
  (9.55,-0.65)--(9.55,0.92);

\draw[black!45,densely dotted]
  (11.4,-0.65)--(11.4,0.86);

\draw[
  paperBlue,
  line width=1pt,
  domain=-1.65:1.65,
  samples=70,
  smooth
]
  plot ({
    11.4+\x
  },{
    -0.65+1.48*exp(-\x*\x/(0.42*0.42))
  });

\draw[
  paperOrange,
  line width=1pt,
  domain=-1.65:1.65,
  samples=70,
  smooth
]
  plot ({
    11.4+\x
  },{
    -0.65+0.6216*exp(-\x*\x)
  });

\node[
  paperBlue,
  font=\scriptsize,
  align=center
] at (10.15,1.12)
  {Original endpoint\\density};

\node[
  paperOrange,
  font=\scriptsize,
  align=center
] at (12.65,1.12)
  {Perturbed endpoint\\density};

\node[
  font=\scriptsize,
  rotate=90
] at (9.25,0.08)
  {Density};

\node[font=\scriptsize] at (11.4,-0.94)
  {Endpoint value};

\end{tikzpicture}

\captionsetup{skip=2.2pt}
\caption{\textbf{From local variation to set-level diversity.}
Three endpoints (blue) are perturbed within centered neighborhoods
(orange). The rightmost curves show the original and perturbed endpoint densities. They have the same endpoint mean, but the perturbed density is broader, which increases the expected pairwise spread.}
\label{fig:pairwise-diversity}
\end{figure}

\subsection{Efficient computation of velocity perturbations}
\label{subsec:compute_direction}
\label{subsec:singular-to-eigen}

In this section, we describe the computation of the velocity perturbation. Recall we need to sample a random direction from the top-$k$ leading right singular subspace of the endpoint Jacobian. Direct computation of the Jacobian is infeasible. Instead, we propose an efficient algorithm inspired by the classic power method from numerical linear algebra \citep{golub2013matrix}. At a latent trajectory state $x_t$, we directly sample the velocity perturbation without explicitly computing the Jacobian $J_t=\nabla F_t(x_t)$, in which $F_t(x_t)=x+(1-t)u_\theta(x_t,t)$ is the one-step predictor of the endpoint.

\textbf{Velocity Perturbation Computation via Power Method.} To sample a velocity perturbation, we draw an isotropic random vector $w$ and multiply it repeatedly by the matrix $J_t^\top J_t$. Take its singular value decomposition $J_t= U\Sigma V^\top
    = \sum_{i=1}^{d} \sigma_i u_i v_i^\top$. We can decompose $w$ into the basis of $v_i$ as $w=\sum_{i=1}^d\langle w,v_i\rangle v_i$. 
    Since $J_t w = \sum_{i=1}^d \sigma_i \langle w, v_i \rangle u_i$, multiplying by $J_t$ scales each component of $w$ along the right singular vector $v_i$ by $\sigma_i$ and maps it onto the left singular vector $u_i$. The subsequent product with $J_t^\top$ then pulls this response back into the input space, mapping the components back onto the right singular vectors via $J_t^\top J_t w = \sum_{i=1}^d \sigma_i^2 \langle w, v_i \rangle v_i$. Consequently, inputs aligned with the leading right singular vectors experience quadratic amplification ($\sigma_i^2$) along their original directions.


We avoid explicitly forming $J_t^\top J_t$ by relying on matrix-free operations. Specifically, we approximate the Jacobian-vector product $J_t w$ via finite differences,
    $J_t w \approx \left[F_t(x_t + \epsilon w) - F_t(x_t)\right]/\epsilon$,
which requires only two forward network evaluations. For the transpose product $J_t^\top (\cdot)$, we compute an exact vector-Jacobian product using reverse-mode automatic differentiation \citep{baydin2018automatic}. When computing the leading $k$-dimensional subspace ($k > 1$), we employ block power iteration over these alternating products. Once converged, the resulting singular vectors form an orthonormal basis $Q \in \mathbb{R}^{d \times k}$. We then sample a random unit vector $r \sim \mathbb{S}^{k-1}$, construct the perturbation $\delta = \varepsilon Q r$, and resume the integration trajectory from the perturbed state $x_t + \delta$.


\textbf{One-Iteration Approximation of Block Power Iteration.}
Computing the full leading singular subspace requires repeated $J_t^\top J_t$ evaluations, which introduces significant computational overhead. To mitigate this cost, we consider two lower-budget approximations for determining the perturbation direction. In both variants, we draw a single isotropic probe $w$ and transform it into a candidate direction $\delta$. We then resume sampling from the perturbed state $x_t + s\delta$, where $s \in \{-1, +1\}$ is a Rademacher sign variable sampled with equal probability. 


\noindent\textbf{\textcolor{paperOrange}{Solution 1: $J_t^\top J_t$ Product.}}
We form $\delta=\varepsilon J_t^\top J_tw/\|J_t^\top J_tw\|_2$, which amounts to a single $J_t^\top J_t$ product iteration. If $w=\sum_i a_iv_i$, then $J_t^\top J_tw=\sum_i\sigma_i^2a_iv_i$ so that the factor $\sigma_i^2$ amplifies components along high gain right singular directions. 

\noindent\textbf{\textcolor{paperBlue}{Solution 2: Jacobian Vector Product (JVP).}}
We form $\delta=\varepsilon J_tw/\|J_tw\|_2$ by estimating $J_tw$
with finite differences. This solution is much cheaper as it only costs two forward evaluations to compute.

As shown in Table~\ref{tab:runtime}, a single JVP is highly efficient, requiring roughly twice the runtime of a standard forward pass. While evaluating $J_t^\top J_t$ doubles this cost, empirical evaluations (Appendix Table~\ref{tab:ek_gain_full}) demonstrate its substantially superior accuracy in approximating the top-$k$ singular subspace. Nonetheless, JVP consistently outperforms isotropic random velocity perturbations. In practice, $J_t^\top J_t$ delivers the most reliable performance, while JVP provides an effective, low-cost alternative.

\begin{table}[htbp]
\centering
\caption{Computational efficiency of the proposed methods ($J_t^\top J_t$ vs. ~JVP) with one matrix iteration. This is evaluated on
FLUX.1-schnell at $512\times512$ resolution on an A100.}
\label{tab:runtime}
\footnotesize
\setlength{\tabcolsep}{3.5pt}
\begin{tabular}{lcccc}
\toprule
Methods & Seconds & Time$\times$ & FLOPs$\times$ & Mem.\,(GB) \\
\midrule
One forward pass & 0.307 & 1.00 & 1.00 & 22.3 \\
\textcolor{paperBlue}{JVP}
& \textcolor{paperBlue}{0.617}
& \textcolor{paperBlue}{2.01}
& \textcolor{paperBlue}{2.00}
& \textcolor{paperBlue}{22.3} \\
\textcolor{paperOrange}{$J_t^\top J_t$}
& \textcolor{paperOrange}{1.302}
& \textcolor{paperOrange}{4.24}
& \textcolor{paperOrange}{5.04}
& \textcolor{paperOrange}{35.4} \\
\bottomrule
\end{tabular}
\end{table}

\section{Experiments}
\label{sec:experiments}

We evaluate JIVE's diversity gains in both its full and low-cost settings, test its combination with an existing method, and apply it to inversion-based generation. We apply our intervention at initial time $t=0$ in all experiments. All finite difference estimates use probe perturbations with norm 4 unless otherwise specified.

\subsection{Benchmark diversity}
\label{sec:parti-results}

We first evaluate whether JIVE increases diversity across prompts
and complements existing diversity guidance. We compare JIVE with OSCAR \citep{wu2025oscar}, which repels samples within a batch, and CADS \citep{sadat2024cads}, which adds noise to text conditioning
and gradually reduces the noise level during sampling.

\begin{table}[!b]
\centering
\captionsetup{skip=2pt}
\caption{\textbf{Diversity and plug-in composition on PartiPrompts.} We evaluate both \textcolor{paperBlue}{JIVE(JVP)} and \textcolor{paperOrange}{JIVE($J_t^\top J_t$)} on FLUX.1-schnell at four evaluation steps. We use prompts from PartiPrompts at 16 images per prompt. Reported values are mean$\pm$SE over prompts. The norms used are 12 on both JIVE variants when alone, and 4 when combined with CADS ($\tau_2=1.2$).}
\label{tab:parti-bands}
\setlength{\tabcolsep}{1.5pt}
\renewcommand{\arraystretch}{0.85}
\setlength{\aboverulesep}{1pt}
\setlength{\belowrulesep}{1pt}
\footnotesize
\begin{tabular*}{\linewidth}{@{\extracolsep{\fill}}lcccccc@{}}
\toprule
Method & F-Vendi$\uparrow$ & P-Vendi$\uparrow$ & Pairwise $L_2\uparrow$ & CLIP$\uparrow$ & CLIP-IQA$\uparrow$ & HPSv2$\uparrow$ \\
\midrule
\multicolumn{7}{l}{\textbf{Standard} ($n=504$)} \\
\vphantom{$(J_t^\top J_t)$}Unmodified sampler & \meanse{3.76}{0.10} & \meanse{5.00}{0.05} & \meanse{323.69}{2.80} & \meanse{30.25}{0.13} & \meanse{0.668}{0.003} & \meanse{28.68}{0.14} \\
\vphantom{$(J_t^\top J_t)$}OSCAR & \meanse{3.80}{0.10} & \meanse{4.98}{0.05} & \meanse{327.43}{2.81} & \meanse{30.27}{0.13} & \meanse{0.664}{0.003} & \meanse{28.66}{0.14} \\
\vphantom{$(J_t^\top J_t)$}\textcolor{paperBlue}{JIVE(JVP)} & \textcolor{paperBlue}{\meanse{4.73}{0.11}} & \textcolor{paperBlue}{\meanse{5.49}{0.03}} & \textcolor{paperBlue}{\meanse{455.83}{2.12}} & \textcolor{paperBlue}{\meanse{30.00}{0.13}} & \textcolor{paperBlue}{\meanse{0.608}{0.003}} & \textcolor{paperBlue}{\meanse{25.98}{0.14}} \\
\vphantom{$(J_t^\top J_t)$}\textcolor{paperOrange}{JIVE($J_t^\top J_t$)} & \textcolor{paperOrange}{\meanse{5.14}{0.11}} & \textcolor{paperOrange}{\textbf{\meanse{5.60}{0.02}}} & \textcolor{paperOrange}{\textbf{\meanse{458.20}{1.75}}} & \textcolor{paperOrange}{\meanse{29.83}{0.12}} & \textcolor{paperOrange}{\meanse{0.602}{0.003}} & \textcolor{paperOrange}{\meanse{25.62}{0.13}} \\
\vphantom{$(J_t^\top J_t)$}CADS ($\tau_2{=}1.2$) & \meanse{6.10}{0.16} & \meanse{5.24}{0.04} & \meanse{344.42}{2.14} & \meanse{28.96}{0.15} & \meanse{0.678}{0.003} & \meanse{27.70}{0.15} \\
\vphantom{$(J_t^\top J_t)$}\textcolor{paperBlue}{JIVE(JVP)+CADS} & \textcolor{paperBlue}{\meanse{6.31}{0.16}} & \textcolor{paperBlue}{\meanse{5.54}{0.04}} & \textcolor{paperBlue}{\meanse{391.89}{2.14}} & \textcolor{paperBlue}{\meanse{28.97}{0.15}} & \textcolor{paperBlue}{\meanse{0.661}{0.003}} & \textcolor{paperBlue}{\meanse{26.87}{0.15}} \\
\vphantom{$(J_t^\top J_t)$}\textcolor{paperOrange}{JIVE($J_t^\top J_t$)+CADS} & \textcolor{paperOrange}{\textbf{\meanse{6.37}{0.16}}} & \textcolor{paperOrange}{\meanse{5.55}{0.03}} & \textcolor{paperOrange}{\meanse{385.30}{1.91}} & \textcolor{paperOrange}{\meanse{28.91}{0.15}} & \textcolor{paperOrange}{\meanse{0.663}{0.003}} & \textcolor{paperOrange}{\meanse{26.92}{0.15}} \\
\midrule
\multicolumn{7}{l}{\textbf{Intermediate} ($n=518$)} \\
\vphantom{$(J_t^\top J_t)$}Unmodified sampler & \meanse{2.95}{0.06} & \meanse{4.65}{0.04} & \meanse{294.19}{2.64} & \meanse{32.88}{0.13} & \meanse{0.679}{0.003} & \meanse{30.07}{0.15} \\
\vphantom{$(J_t^\top J_t)$}OSCAR & \meanse{2.96}{0.06} & \meanse{4.63}{0.04} & \meanse{297.08}{2.67} & \meanse{32.93}{0.13} & \meanse{0.675}{0.003} & \meanse{30.10}{0.14} \\
\vphantom{$(J_t^\top J_t)$}\textcolor{paperBlue}{JIVE(JVP)} & \textcolor{paperBlue}{\meanse{3.61}{0.07}} & \textcolor{paperBlue}{\meanse{5.27}{0.03}} & \textcolor{paperBlue}{\meanse{412.30}{2.67}} & \textcolor{paperBlue}{\meanse{32.33}{0.13}} & \textcolor{paperBlue}{\meanse{0.631}{0.003}} & \textcolor{paperBlue}{\meanse{27.66}{0.16}} \\
\vphantom{$(J^\top J)$}\textcolor{paperOrange}{JIVE($J_t^\top J_t$)} & \textcolor{paperOrange}{\meanse{3.99}{0.07}} & \textcolor{paperOrange}{\textbf{\meanse{5.47}{0.03}}} & \textcolor{paperOrange}{\textbf{\meanse{430.17}{2.01}}} & \textcolor{paperOrange}{\meanse{32.09}{0.12}} & \textcolor{paperOrange}{\meanse{0.620}{0.004}} & \textcolor{paperOrange}{\meanse{27.28}{0.14}} \\
\vphantom{$(J_t^\top J_t)$}CADS ($\tau_2{=}1.2$) & \meanse{4.32}{0.08} & \meanse{5.11}{0.04} & \meanse{330.63}{2.31} & \meanse{31.72}{0.13} & \meanse{0.687}{0.003} & \meanse{29.29}{0.14} \\
\vphantom{$(J_t^\top J_t)$}\textcolor{paperBlue}{JIVE(JVP)+CADS} & \textcolor{paperBlue}{\meanse{4.42}{0.08}} & \textcolor{paperBlue}{\meanse{5.33}{0.04}} & \textcolor{paperBlue}{\meanse{365.81}{2.50}} & \textcolor{paperBlue}{\meanse{31.64}{0.13}} & \textcolor{paperBlue}{\meanse{0.674}{0.003}} & \textcolor{paperBlue}{\meanse{28.64}{0.14}} \\
\vphantom{$(J^\top J)$}\textcolor{paperOrange}{JIVE($J_t^\top J_t$)+CADS} & \textcolor{paperOrange}{\textbf{\meanse{4.46}{0.08}}} & \textcolor{paperOrange}{\meanse{5.41}{0.04}} & \textcolor{paperOrange}{\meanse{369.10}{2.31}} & \textcolor{paperOrange}{\meanse{31.56}{0.13}} & \textcolor{paperOrange}{\meanse{0.676}{0.003}} & \textcolor{paperOrange}{\meanse{28.72}{0.14}} \\
\midrule
\multicolumn{7}{l}{\textbf{Challenging} ($n=610$)} \\
\vphantom{$(J_t^\top J_t)$}Unmodified sampler & \meanse{3.03}{0.06} & \meanse{4.74}{0.04} & \meanse{311.35}{2.51} & \meanse{33.22}{0.13} & \meanse{0.665}{0.004} & \meanse{30.14}{0.14} \\
\vphantom{$(J_t^\top J_t)$}OSCAR & \meanse{3.06}{0.06} & \meanse{4.72}{0.04} & \meanse{314.71}{2.53} & \meanse{33.27}{0.13} & \meanse{0.660}{0.004} & \meanse{30.15}{0.14} \\
\vphantom{$(J_t^\top J_t)$}\textcolor{paperBlue}{JIVE(JVP)} & \textcolor{paperBlue}{\meanse{3.78}{0.07}} & \textcolor{paperBlue}{\meanse{5.35}{0.03}} & \textcolor{paperBlue}{\meanse{427.71}{2.47}} & \textcolor{paperBlue}{\meanse{32.74}{0.13}} & \textcolor{paperBlue}{\meanse{0.622}{0.004}} & \textcolor{paperBlue}{\meanse{27.90}{0.15}} \\
\vphantom{$(J_t^\top J_t)$}\textcolor{paperOrange}{JIVE($J_t^\top J_t$)} & \textcolor{paperOrange}{\meanse{4.14}{0.07}} & \textcolor{paperOrange}{\textbf{\meanse{5.52}{0.03}}} & \textcolor{paperOrange}{\textbf{\meanse{437.78}{1.98}}} & \textcolor{paperOrange}{\meanse{32.59}{0.13}} & \textcolor{paperOrange}{\meanse{0.614}{0.004}} & \textcolor{paperOrange}{\meanse{27.55}{0.14}} \\
\vphantom{$(J_t^\top J_t)$}CADS ($\tau_2{=}1.2$) & \meanse{4.88}{0.09} & \meanse{5.23}{0.04} & \meanse{339.85}{2.17} & \meanse{31.66}{0.13} & \meanse{0.673}{0.003} & \meanse{28.79}{0.14} \\
\vphantom{$(J_t^\top J_T)$}\textcolor{paperBlue}{JIVE(JVP)+CADS} & \textcolor{paperBlue}{\meanse{4.92}{0.09}} & \textcolor{paperBlue}{\meanse{5.44}{0.03}} & \textcolor{paperBlue}{\meanse{380.16}{2.23}} & \textcolor{paperBlue}{\meanse{31.64}{0.12}} & \textcolor{paperBlue}{\meanse{0.659}{0.003}} & \textcolor{paperBlue}{\meanse{28.13}{0.14}} \\
\vphantom{$(J_t^\top J_t)$}\textcolor{paperOrange}{JIVE($J_t^\top J_t$)+CADS} & \textcolor{paperOrange}{\textbf{\meanse{5.01}{0.09}}} & \textcolor{paperOrange}{\meanse{5.52}{0.03}} & \textcolor{paperOrange}{\meanse{378.32}{2.01}} & \textcolor{paperOrange}{\meanse{31.62}{0.12}} & \textcolor{paperOrange}{\meanse{0.663}{0.003}} & \textcolor{paperOrange}{\meanse{28.21}{0.14}} \\
\midrule
\multicolumn{7}{l}{\textbf{All} ($n=1632$)} \\
\vphantom{$(J^\top J)$}Unmodified sampler & \meanse{3.23}{0.04} & \meanse{4.79}{0.03} & \meanse{309.72}{1.55} & \meanse{32.19}{0.08} & \meanse{0.671}{0.002} & \meanse{29.66}{0.08} \\
\vphantom{$(J^\top J)$}OSCAR & \meanse{3.26}{0.04} & \meanse{4.77}{0.03} & \meanse{313.05}{1.56} & \meanse{32.23}{0.08} & \meanse{0.666}{0.002} & \meanse{29.67}{0.08} \\
\vphantom{$(J^\top J)$}\textcolor{paperBlue}{JIVE(JVP)} & \textcolor{paperBlue}{\meanse{4.02}{0.05}} & \textcolor{paperBlue}{\meanse{5.37}{0.02}} & \textcolor{paperBlue}{\meanse{431.50}{1.48}} & \textcolor{paperBlue}{\meanse{31.76}{0.08}} & \textcolor{paperBlue}{\meanse{0.621}{0.002}} & \textcolor{paperBlue}{\meanse{27.23}{0.09}} \\
\vphantom{$(J^\top J)$}\textcolor{paperOrange}{JIVE($J_t^\top J_t$)} & \textcolor{paperOrange}{\meanse{4.40}{0.05}} & \textcolor{paperOrange}{\textbf{\meanse{5.53}{0.01}}} & \textcolor{paperOrange}{\textbf{\meanse{441.67}{1.15}}} & \textcolor{paperOrange}{\meanse{31.58}{0.08}} & \textcolor{paperOrange}{\meanse{0.612}{0.002}} & \textcolor{paperOrange}{\meanse{26.87}{0.08}} \\
\vphantom{$(J^\top J)$}CADS ($\tau_2{=}1.2$) & \meanse{5.08}{0.07} & \meanse{5.20}{0.02} & \meanse{338.34}{1.28} & \meanse{30.84}{0.08} & \meanse{0.679}{0.002} & \meanse{28.61}{0.09} \\
\vphantom{$(J^\top J)$}\textcolor{paperBlue}{JIVE(JVP)+CADS} & \textcolor{paperBlue}{\meanse{5.19}{0.07}} & \textcolor{paperBlue}{\meanse{5.44}{0.02}} & \textcolor{paperBlue}{\meanse{379.23}{1.35}} & \textcolor{paperBlue}{\meanse{30.82}{0.08}} & \textcolor{paperBlue}{\meanse{0.664}{0.002}} & \textcolor{paperBlue}{\meanse{27.90}{0.09}} \\
\vphantom{$(J^\top J)$}\textcolor{paperOrange}{JIVE($J_t^\top J_t$)+CADS} & \textcolor{paperOrange}{\textbf{\meanse{5.26}{0.07}}} & \textcolor{paperOrange}{\meanse{5.49}{0.02}} & \textcolor{paperOrange}{\meanse{377.55}{1.22}} & \textcolor{paperOrange}{\meanse{30.76}{0.08}} & \textcolor{paperOrange}{\meanse{0.667}{0.002}} & \textcolor{paperOrange}{\meanse{27.97}{0.08}} \\
\bottomrule
\end{tabular*}
\end{table}

\noindent \textbf{Metrics.}
We measure diversity using the Vendi score \citep{friedman2023the}
in both DINOv2 \citep{oquab2024dinov} feature space (F-Vendi) and pixel space (P-Vendi), together
with mean pairwise $L_2$ distance in pixel space. We evaluate prompt
alignment with CLIPScore \citep{hessel2021clipscore}, image quality
with CLIP-IQA \citep{wang2022exploringclipassessinglook}, and human
preference alignment with HPSv2 \citep{wu2023human}.

\noindent \textbf{Setup.}
We generate 16 images for each of the 1{,}632 PartiPrompts prompts
\citep{yu2022scaling} using four-step FLUX.1-schnell
\citep{blackforestlabs2024flux}, with matched initial latents across
methods. Both \textcolor{paperBlue}{JIVE(JVP)} and \textcolor{paperOrange}{JIVE($J_t^\top J_t$)} estimate a four-dimensional subspace
using 10 iterations,
and apply a single perturbation of norm 12 before sampling.
Table~\ref{tab:parti-bands} reports the effectiveness of our method across different prompt difficulties, along with their aggregate score.

\noindent \textbf{Observation 1: JIVE improves diversity over OSCAR.} Both JIVE variants improve F-Vendi, P-Vendi, and pairwise $L_2$ across every prompt difficulty level, with modestly lower text alignment and a larger cost in image quality and preference. Relative to CADS, standalone JIVE yields greater pixel-space diversity but lower Feature-space Vendi, motivating their combination which we explore in the JIVE+CADS rows.
\par

\noindent \textbf{Observation 2: JIVE adds diversity to CADS.} Adding JIVE to CADS with a smaller perturbation norm of 4 improves all three diversity measures over CADS in every band, with small decreases in quality and alignment. This shows that JIVE composes effectively with existing diversity guidance and can serve as a flexible plugin method.
\par

\subsection{Low cost direction estimation}
\label{sec:low_cost}

One-step generation has become increasingly important for practical generative models \citep{geng2026mean,zhang2026stable,geng2026improved}. However, this efficiency often comes at the cost of diversity \citep{zhang2026teacherfeaturedriftingonestepdiffusion}.  We therefore evaluate JIVE on one-step FLUX.1-schnell using a single estimated direction and a single estimation iteration. We also test this minimal
budget variant on four-step SD3.5-Large-Turbo
\citep{sauer2024fast,esser2024scaling}, as an additional model family for comparison. Figure~\ref{fig:geneval-frontiers} compares the diversity-fidelity tradeoff curves of both JIVE variants, CADS, and OSCAR, and Table \ref{tab:method-comparison} reports representative norms from these curves. Both JVP and $J_t^\top J_t$ variants use a finite difference for the Jacobian-vector product. For SD3.5-Large-Turbo, we use a finite difference probe norm of 10 due to the higher latent space dimension.

We evaluate on 553 GenEval prompts \citep{ghosh2023geneval}, generating
four images per prompt. We
report additional fidelity metrics which are GenEval score for compositional prompt adherence, and PickScore
\citep{kirstain2023pickapic} for human preference alignment.
Appendix~\ref{app:direction-ablations} provides the corresponding
tradeoff curves for GenEval and PickScore.

\begin{figure}[!b]
\centering
\includegraphics[width=\linewidth]{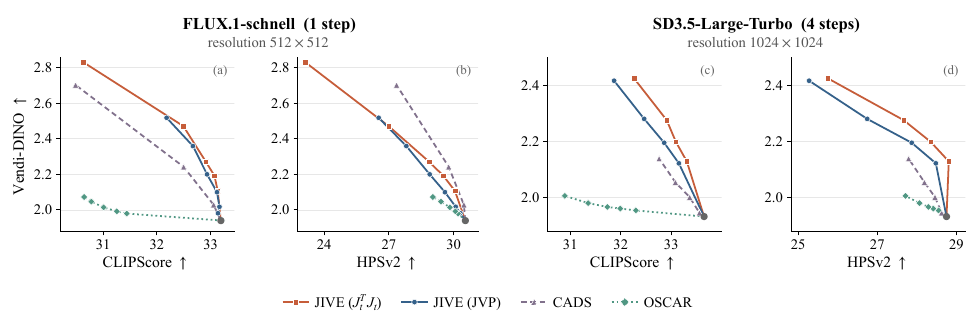}
\caption{\textbf{Diversity-fidelity tradeoffs on GenEval.} We show tradeoff curves for CLIPScore and HPSv2 versus Vendi-DINO, averaged over all 553 GenEval prompts at 4 images per prompt. Additional tradeoff curves for PickScore and GenEval score are in Appendix Figure \ref{fig:geneval-appendix-frontiers}.}
\label{fig:geneval-frontiers}
\end{figure}

\begin{table}[h]
\centering
\caption{Diversity-fidelity comparison on GenEval
(553 prompts, 4 images per prompt).
Both JIVE variants use one direction-estimation iteration with finite differences. JVP/$J_t^\top J_t$ use norms 26/8 on FLUX and 26/16 on Turbo, respectively.  Bold indicates the highest diversity score within each model.}
\label{tab:method-comparison}
\small
\setlength{\tabcolsep}{5pt}
\newcommand{\partise}[2]{#1{\fontsize{6}{7}\selectfont$\pm$#2}}
\newcommand{\jtablecell}[1]{\textcolor{paperBlue}{#1}}
\newcommand{\jtjtablecell}[1]{\textcolor{paperOrange}{#1}}
\begin{tabular}{lrrrrrc}
\toprule
Method & F-Vendi $\uparrow$ & P-Vendi $\uparrow$ & CLIP $\uparrow$ & HPSv2 $\uparrow$ & PickScore $\uparrow$ & GenEval $\uparrow$ \\
\midrule
\multicolumn{6}{l}{\textbf{FLUX.1-schnell}
\quad (1 step, $512\times512$)} \\
\midrule
Unmodified sampler
& \partise{1.940}{0.019} & \partise{1.683}{0.013} & \partise{33.189}{0.138} & \partise{30.540}{0.109} & \partise{23.330}{0.044} & 0.700 \\
CADS ($\tau_2{=}1.2$)
& \partise{2.243}{0.024} & \partise{1.779}{0.013} & \partise{32.497}{0.134} & \partise{29.770}{0.120} & \partise{23.026}{0.043} & 0.583 \\
OSCAR
& \partise{2.073}{0.021} & \partise{1.680}{0.013} & \partise{30.648}{0.113} & \partise{29.037}{0.118} & \partise{22.776}{0.045} & 0.667 \\
\jtablecell{\textbf{Ours} (JVP, FD)}
& \jtablecell{\partise{2.200}{0.021}} & \jtablecell{\partise{1.964}{0.018}} & \jtablecell{\partise{32.932}{0.132}} & \jtablecell{\partise{28.890}{0.119}} & \jtablecell{\partise{22.945}{0.044}} & \jtablecell{0.684} \\
\jtjtablecell{\textbf{Ours} ($J_t^\top J_t$, FD)}
& \jtjtablecell{\textbf{\partise{2.272}{0.023}}} & \jtjtablecell{\textbf{\partise{2.080}{0.016}}} & \jtjtablecell{\partise{32.909}{0.132}} & \jtjtablecell{\partise{28.858}{0.115}} & \jtjtablecell{\partise{22.910}{0.043}} & \jtjtablecell{0.672} \\
\midrule
\multicolumn{6}{l}{\textbf{SD3.5-Large-Turbo}
\quad (4 steps, $1024\times1024$)} \\
\midrule
Unmodified sampler
& \partise{1.932}{0.021} & \partise{1.369}{0.007} & \partise{33.650}{0.127} & \partise{28.751}{0.116} & \partise{23.112}{0.043} & 0.708 \\
CADS ($\tau_2{=}1.1$)
& \partise{2.139}{0.025} & \partise{1.370}{0.007} & \partise{32.760}{0.114} & \partise{27.793}{0.119} & \partise{22.824}{0.041} & 0.582 \\
OSCAR
& \partise{2.006}{0.022} & \partise{1.367}{0.008} & \partise{30.892}{0.111} & \partise{27.713}{0.121} & \partise{22.783}{0.042} & 0.701 \\
\jtablecell{\textbf{Ours} (JVP, FD)}
& \jtablecell{\partise{2.195}{0.023}} & \jtablecell{\partise{1.964}{0.015}} & \jtablecell{\partise{32.860}{0.125}} & \jtablecell{\partise{27.871}{0.119}} & \jtablecell{\partise{22.677}{0.043}} & \jtablecell{0.657} \\
\jtjtablecell{\textbf{Ours} ($J_t^\top J_t$, FD)}
& \jtjtablecell{\textbf{\partise{2.276}{0.023}}} & \jtjtablecell{\textbf{\partise{1.977}{0.013}}} & \jtjtablecell{\partise{32.910}{0.126}} & \jtjtablecell{\partise{27.667}{0.121}} & \jtjtablecell{\partise{22.638}{0.041}} & \jtjtablecell{0.674} \\
\bottomrule
\end{tabular}
\end{table}

\noindent\textbf{Observation 3: $J_t^\top J_t$ gives a stronger tradeoff than JVP.}
Across both models, $J_t^\top J_t$ generally achieves higher Vendi-DINO at comparable CLIPScore or HPSv2. This is consistent with $J_t^\top J_t$ placing more perturbation energy in high-gain right singular directions, thereby producing greater endpoint diversity for a similar quality cost. The JVP variant remains a useful cheaper alternative for finding a perturbation direction.

\noindent \textbf{Observation 4: JIVE remains effective in one-step generation.} On one-step FLUX.1-schnell, JIVE($J_t^\top J_t$) achieves a more favorable diversity-CLIPScore and GenEval tradeoff than CADS. CADS retains a stronger HPSv2 tradeoff and remains competitive under PickScore, showing that the relative advantage depends on the quality metric. Both JIVE variants offer stronger diversity-quality tradeoffs than OSCAR over the evaluated settings.

\noindent \textbf{Observation 5: JIVE transfers effectively to SD3.5-Large-Turbo.}
On four-step SD3.5-Large-Turbo, JIVE($J_t^\top J_t$) exhibits the strongest observed diversity-quality tradeoff across CLIPScore, HPSv2, PickScore, and GenEval. The JVP variant also generally improves on CADS and OSCAR while requiring less computation for direction estimation. Compared with FLUX, the advantage over CADS is more consistent across quality metrics, supporting the effectiveness of JIVE's minimal estimation budget on a second generator.

\begin{figure}[!b]
    \centering
    \input{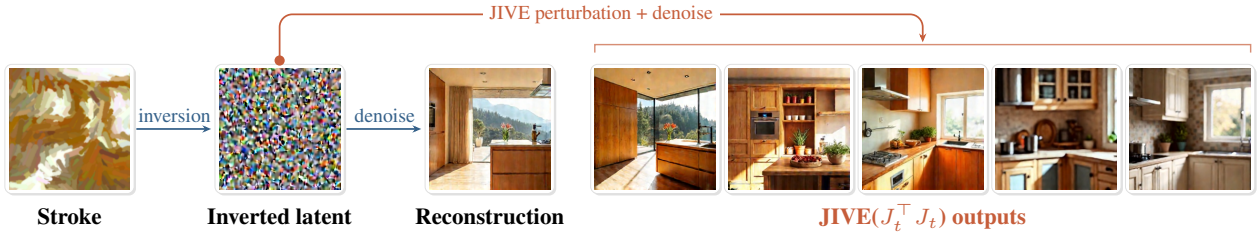}
    \caption{\textbf{Diverse generation from inverted latents.} Controlled ODE inversion maps a stroke image to a latent. Denoising this latent returns a single reconstruction, while applying a JIVE($J^{\top}J$) perturbation after the first denoising step yields distinct outputs from the same latent.}
    \label{fig:rf_qualitative}
\end{figure}

\subsection{Improving diversity in inversion-based generation}
\label{sec:rf_inversion}

Many image-generation tasks start from an existing image rather than generating from scratch. The goal is to produce new images that remain related to the source while allowing meaningful variations. Diversity is important in this setting because a single source can admit many plausible outputs, and generating multiple alternatives gives users or downstream systems more choices. We consider inversion-based generation \citep{meng2022sdedit,mokady2022nulltextinversioneditingreal,luzi2024boomerang}, where the starting latent is produced by inverting a source image rather than drawing from Gaussian noise. We study this setting by combining JIVE with RF-Inversion \citep{rout2024semantic} on stroke-to-image tasks. Given a paint-stroke image, RF-Inversion inverts a controlled ODE to recover a starting latent, then integrates forward to generate an image. During denoising, we inject JIVE($J_t^\top J_t$) after the first reverse-step controller update ($\eta_t=0$). On FLUX.1-dev, we use 32 source images for paint strokes and generate 16 outputs per source for various different scenes from the LSUN dataset \citep{yu2015lsun}. Table~\ref{tab:stroke2image-rf-inverse-aftereta-lsun4} reports diversity and fidelity and Figure~\ref{fig:rf_qualitative} illustrates the workflow.


\begin{table}[!htb]
\centering
\setlength{\tabcolsep}{3.5pt}
\caption{\textbf{Diverse edits from inverted latents.} RF-Inversion on FLUX.1-dev (30 steps), with JIVE($J_t^\top J_t$) injected after the first reverse-step controller update ($\eta_t=0$, norm 8). F-Vendi uses DINOv2-base features while P-Vendi uses pixels. Source $L_2$ is measured relative to the source stroke image. Values are mean$\pm$SE over 32 sources. Orange denotes JIVE.}
\label{tab:stroke2image-rf-inverse-aftereta-lsun4}
\resizebox{\textwidth}{!}{%
\definecolor{paperOrange}{RGB}{204,85,0}
\newcommand{\jiveorcell}[1]{\textcolor{paperOrange}{#1}}
\begin{tabular}{llcccccccc}
\toprule
Dataset & Method & Controller & Norm & F-Vendi$\uparrow$ & P-Vendi$\uparrow$
 & LPIPS$_d$$\uparrow$ & Source $L_2\downarrow$ & DINO$_f$$\uparrow$
 & KID$\times$1k$\downarrow$ \\
\midrule
classroom & baseline & $\eta\!=\!0.5$ & \text{N/A} & \meanse{4.67}{0.17} & \meanse{2.82}{0.06} & \meanse{0.684}{0.002} & \meanse{121.68}{1.34} & \meanse{0.051}{0.002} & 53.95 \\
 & \jiveorcell{JIVE($J_t^{\top}J_t$)} & \jiveorcell{$\eta\!=\!0.5$} & \jiveorcell{8} & \jiveorcell{\textbf{\meanse{5.37}{0.16}}} & \jiveorcell{\textbf{\meanse{2.98}{0.06}}} & \jiveorcell{\textbf{\meanse{0.716}{0.002}}} & \jiveorcell{\meanse{126.12}{1.34}} & \jiveorcell{\meanse{0.057}{0.002}} & \jiveorcell{54.90} \\
\midrule
kitchen & baseline & $\eta\!=\!0.5$ & \text{N/A} & \meanse{6.78}{0.18} & \meanse{2.33}{0.07} & \meanse{0.680}{0.003} & \meanse{114.09}{0.96} & \meanse{0.018}{0.003} & 40.62 \\
 & \jiveorcell{JIVE($J_t^{\top}J_t$)} & \jiveorcell{$\eta\!=\!0.5$} & \jiveorcell{8} & \jiveorcell{\textbf{\meanse{7.40}{0.16}}} & \jiveorcell{\textbf{\meanse{2.46}{0.07}}} & \jiveorcell{\textbf{\meanse{0.723}{0.002}}} & \jiveorcell{\meanse{121.47}{1.02}} & \jiveorcell{\meanse{0.023}{0.003}} & \jiveorcell{39.99} \\
\midrule
conference room & baseline & $\eta\!=\!0.5$ & \text{N/A} & \meanse{4.99}{0.16} & \meanse{2.69}{0.11} & \meanse{0.654}{0.002} & \meanse{119.87}{1.60} & \meanse{0.085}{0.004} & 35.12 \\
 & \jiveorcell{JIVE($J_t^{\top}J_t$)} & \jiveorcell{$\eta\!=\!0.5$} & \jiveorcell{8} & \jiveorcell{\textbf{\meanse{5.80}{0.18}}} & \jiveorcell{\textbf{\meanse{2.88}{0.12}}} & \jiveorcell{\textbf{\meanse{0.696}{0.003}}} & \jiveorcell{\meanse{126.15}{1.67}} & \jiveorcell{\meanse{0.086}{0.003}} & \jiveorcell{25.50} \\
\midrule
dining room & baseline & $\eta\!=\!0.5$ & \text{N/A} & \meanse{6.92}{0.14} & \meanse{2.62}{0.09} & \meanse{0.652}{0.003} & \meanse{118.00}{1.15} & \meanse{0.065}{0.003} & 11.49 \\
 & \jiveorcell{JIVE($J_t^{\top}J_t$)} & \jiveorcell{$\eta\!=\!0.5$} & \jiveorcell{8} & \jiveorcell{\textbf{\meanse{7.93}{0.12}}} & \jiveorcell{\textbf{\meanse{2.80}{0.09}}} & \jiveorcell{\textbf{\meanse{0.700}{0.003}}} & \jiveorcell{\meanse{124.32}{1.16}} & \jiveorcell{\meanse{0.066}{0.003}} & \jiveorcell{9.52} \\
\midrule
restaurant & baseline & $\eta\!=\!0.5$ & \text{N/A} & \meanse{6.78}{0.25} & \meanse{4.01}{0.15} & \meanse{0.637}{0.004} & \meanse{120.33}{1.91} & \meanse{0.048}{0.004} & 44.59 \\
 & \jiveorcell{JIVE($J_t^{\top}J_t$)} & \jiveorcell{$\eta\!=\!0.5$} & \jiveorcell{8} & \jiveorcell{\textbf{\meanse{7.48}{0.24}}} & \jiveorcell{\textbf{\meanse{4.15}{0.14}}} & \jiveorcell{\textbf{\meanse{0.681}{0.003}}} & \jiveorcell{\meanse{124.36}{2.06}} & \jiveorcell{\meanse{0.047}{0.004}} & \jiveorcell{39.23} \\
\bottomrule
\end{tabular}%
}
\end{table}

\noindent\textbf{Observation 6: diversity gains after inversion often accompany improved quality.}
JIVE($J_t^\top J_t$) improves F-Vendi, P-Vendi, and LPIPS \citep{zhang2018unreasonable} diversity across all five datasets. KID generally improves on these datasets. These results suggest that JIVE explores meaningful alternative reconstructions rather than introducing arbitrary distortion. The outputs move farther from the source and become more diverse while generally preserving or improving quality.

\section{Related work}
\label{sec:related-work}

\noindent \textbf{Inference-time diversity.}
OSCAR, DiverseFlow, and Particle Guidance promote diversity through
repulsive interactions among sampling trajectories
\citep{wu2025oscar,morshed2025diverseflow,corso2024particle}.
Their diversity updates depend on other samples, whereas JIVE selects
each perturbation independently using the generator's Jacobian.
Noise optimization \citep{harrington2026s} iteratively optimizes
initial latents using output-diversity
objectives and quality rewards. It supports both joint
batch optimization and sequential generation, where each new output
is optimized relative to previous outputs. JIVE's cheapest variant
constructs a direction with a single Jacobian product and requires
no auxiliary feature or reward model.

CADS \citep{sadat2024cads} adds noise to text conditioning and gradually
anneals it away during sampling. JIVE instead perturbs the sampling
latent and can be combined with CADS, as demonstrated in our experiments.
STRIDE \citep{yadav2026stride} uses activation PCA to perturb
intermediate transformer features at architecture-specific injection
layers. JIVE selects latent perturbations through the
endpoint-predictor Jacobian without modifying internal features or
selecting injection layers. 
LOCO Edit \citep{chen2024exploring} uses leading Jacobian singular directions for localized image editing. By contrast, our goal is to use the Jacobian to find diversity-inducing velocity perturbations.



\section{Conclusion}
\label{sec:conclusion}

JIVE uses the generator's Jacobian to identify velocity perturbations that alter the trajectory's endpoint. Our analysis connects our local perturbation criterion to
greater expected spread among independently generated outputs. 
Across the evaluated generators, JIVE improves diversity and remains
effective in low cost settings. Its application to inversion-based generation
further demonstrates its usefulness in applications outside of text-to-image generation.

\noindent \textbf{Limitations and future work}
Although we provide a theoretical bound on the perturbation norm for latent distributions and test its effectiveness in our experiment, a general norm-selection principle for latents remains open. Developing such a principle is an important direction for future work. Generation diversity also has applications in robotics, particularly generating diverse action sequences and candidate motion plans for downstream selection, and in other domains such as data augmentation.

\FloatBarrier
\bibliography{iclr2027_conference}

@article{ho2020denoising,
  title={Denoising diffusion probabilistic models},
  author={Ho, Jonathan and Jain, Ajay and Abbeel, Pieter},
  journal={Advances in neural information processing systems},
  volume={33},
  pages={6840--6851},
  year={2020}
}

@inproceedings{
song2020score,
title={Score-Based Generative Modeling through Stochastic Differential Equations},
author={Yang Song and Jascha Sohl-Dickstein and Diederik P Kingma and Abhishek Kumar and Stefano Ermon and Ben Poole},
booktitle={International Conference on Learning Representations},
year={2021},
url={https://openreview.net/forum?id=PxTIG12RRHS}
}

@inproceedings{
lipman2023flow,
title={Flow Matching for Generative Modeling},
author={Yaron Lipman and Ricky T. Q. Chen and Heli Ben-Hamu and Maximilian Nickel and Matthew Le},
booktitle={The Eleventh International Conference on Learning Representations },
year={2023},
url={https://openreview.net/forum?id=PqvMRDCJT9t}
}

@inproceedings{
liu2023flow,
title={Flow Straight and Fast: Learning to Generate and Transfer Data with Rectified Flow},
author={Xingchao Liu and Chengyue Gong and Qiang Liu},
booktitle={The Eleventh International Conference on Learning Representations },
year={2023},
url={https://openreview.net/forum?id=XVjTT1nw5z}
}

@inproceedings{esser2024scaling,
  title={Scaling rectified flow transformers for high-resolution image synthesis},
  author={Esser, Patrick and Kulal, Sumith and Blattmann, Andreas and Entezari, Rahim and M{\"u}ller, Jonas and Saini, Harry and Levi, Yam and Lorenz, Dominik and Sauer, Axel and Boesel, Frederic and Podell, Dustin and Dockhorn, Tim and English, Zion and Rombach, Robin},
  booktitle={Forty-first international conference on machine learning},
  year={2024}
}

@misc{blackforestlabs2024flux,
title={Flux},
author={{Black Forest Labs}},
year={2024},
howpublished = {\url{https://github.com/black-forest-labs/flux}}
}

@inproceedings{sauer2024fast,
  title={Fast high-resolution image synthesis with latent adversarial diffusion distillation},
  author={Sauer, Axel and Boesel, Frederic and Dockhorn, Tim and Blattmann, Andreas and Esser, Patrick and Rombach, Robin},
  booktitle={SIGGRAPH Asia 2024 Conference Papers},
  pages={1--11},
  year={2024}
}

@inproceedings{
wu2025oscar,
title={Letting Trajectories Spread: Quality-Preserving Control for Diverse Flow Matching},
author={Jingxuan Wu and Zhenglin Wan and Xingrui Yu and Yuzhe Yang and Bo An and Ivor Tsang and Yang You},
booktitle={Forty-third International Conference on Machine Learning},
year={2026},
url={https://openreview.net/forum?id=LtHNTRg73G}
}

@article{yadav2026stride,
  title={STRIDE: Training-Free Diversity Guidance via PCA-Directed Feature Perturbation in Single-Step Diffusion Models},
  author={Yadav, Ankit and Garg, Arpit and Huy, Ta Duc and Liu, Lingqiao},
  journal={arXiv preprint arXiv:2605.11494},
  year={2026}
}

@inproceedings{
sadat2024cads,
title={{CADS}: Unleashing the Diversity of Diffusion Models through Condition-Annealed Sampling},
author={Seyedmorteza Sadat and Jakob Buhmann and Derek Bradley and Otmar Hilliges and Romann M. Weber},
booktitle={The Twelfth International Conference on Learning Representations},
year={2024},
url={https://openreview.net/forum?id=zMoNrajk2X}
}

@inproceedings{morshed2025diverseflow,
  title={Diverseflow: Sample-efficient diverse mode coverage in flows},
  author={Morshed, Mashrur M and Boddeti, Vishnu},
  booktitle={2025 IEEE/CVF Conference on Computer Vision and Pattern Recognition (CVPR)},
  pages={23303--23312},
  year={2025},
  organization={IEEE}
}

@inproceedings{
corso2024particle,
title={Particle Guidance: non-I.I.D. Diverse Sampling with Diffusion Models},
author={Gabriele Corso and Yilun Xu and Valentin De Bortoli and Regina Barzilay and Tommi S. Jaakkola},
booktitle={The Twelfth International Conference on Learning Representations},
year={2024},
url={https://openreview.net/forum?id=KqbCvIFBY7}
}

@inproceedings{gandikota2025distillingdiversitycontroldiffusion,
  title={Distilling diversity and control in diffusion models},
  author={Gandikota, Rohit and Bau, David},
  booktitle={2026 IEEE/CVF Winter Conference on Applications of Computer Vision (WACV)},
  pages={1304--1313},
  year={2026},
  organization={IEEE}
}

@inproceedings{
rout2024semantic,
title={Semantic Image Inversion and Editing using Rectified Stochastic Differential Equations},
author={Litu Rout and Yujia Chen and Nataniel Ruiz and Constantine Caramanis and Sanjay Shakkottai and Wen-Sheng Chu},
booktitle={The Thirteenth International Conference on Learning Representations},
year={2025},
url={https://openreview.net/forum?id=Hu0FSOSEyS}
}

@article{
friedman2023the,
title={The Vendi Score: A Diversity Evaluation Metric for Machine Learning},
author={Dan Friedman and Adji Bousso Dieng},
journal={Transactions on Machine Learning Research},
issn={2835-8856},
year={2023},
url={https://openreview.net/forum?id=g97OHbQyk1},
note={}
}

@article{
yu2022scaling,
title={Scaling Autoregressive Models for Content-Rich Text-to-Image Generation},
author={Jiahui Yu and Yuanzhong Xu and Jing Yu Koh and Thang Luong and Gunjan Baid and Zirui Wang and Vijay Vasudevan and Alexander Ku and Yinfei Yang and Burcu Karagol Ayan and Ben Hutchinson and Wei Han and Zarana Parekh and Xin Li and Han Zhang and Jason Baldridge and Yonghui Wu},
journal={Transactions on Machine Learning Research},
issn={2835-8856},
year={2022},
url={https://openreview.net/forum?id=AFDcYJKhND},
note={Featured Certification}
}

@inproceedings{hessel2021clipscore,
  title={Clipscore: A reference-free evaluation metric for image captioning},
  author={Hessel, Jack and Holtzman, Ari and Forbes, Maxwell and Le Bras, Ronan and Choi, Yejin},
  booktitle={Proceedings of the 2021 conference on empirical methods in natural language processing},
  pages={7514--7528},
  year={2021}
}

@inproceedings{wang2022exploringclipassessinglook,
  title={Exploring clip for assessing the look and feel of images},
  author={Wang, Jianyi and Chan, Kelvin CK and Loy, Chen Change},
  booktitle={Proceedings of the AAAI conference on artificial intelligence},
  volume={37},
  number={2},
  pages={2555--2563},
  year={2023}
}

@article{wu2023human,
  title={Human preference score v2: A solid benchmark for evaluating human preferences of text-to-image synthesis},
  author={Wu, Xiaoshi and Hao, Yiming and Sun, Keqiang and Chen, Yixiong and Zhu, Feng and Zhao, Rui and Li, Hongsheng},
  journal={arXiv preprint arXiv:2306.09341},
  year={2023}
}

@article{kirstain2023pickapic,
  title={Pick-a-pic: An open dataset of user preferences for text-to-image generation},
  author={Kirstain, Yuval and Polyak, Adam and Singer, Uriel and Matiana, Shahbuland and Penna, Joe and Levy, Omer},
  journal={Advances in neural information processing systems},
  volume={36},
  pages={36652--36663},
  year={2023}
}

@article{ghosh2023geneval,
  title={Geneval: An object-focused framework for evaluating text-to-image alignment},
  author={Ghosh, Dhruba and Hajishirzi, Hannaneh and Schmidt, Ludwig},
  journal={Advances in Neural Information Processing Systems},
  volume={36},
  pages={52132--52152},
  year={2023}
}

@inproceedings{harrington2026s,
  title={It's Never Too Late: Noise Optimization for Collapse Recovery in Trained Diffusion Models},
  author={Harrington, Anne and Koepke, A. Sophia and Karthik, Shyamgopal and Darrell, Trevor and Efros, Alexei A},
  booktitle={Proceedings of the IEEE/CVF Conference on Computer Vision and Pattern Recognition},
  pages={43124--43134},
  year={2026}
}

@inproceedings{zhang2026stable,
  title={Stable Mean Flow: Lyapunov-Inspired One-Step Flow Matching},
  author={Zhang, Guangxun and Haberle, Mason and Geiger, Davi},
  booktitle={Proceedings of the IEEE/CVF Conference on Computer Vision and Pattern Recognition},
  pages={9223--9232},
  year={2026}
}

@inproceedings{
geng2026mean,
title={Mean Flows for One-step Generative Modeling},
author={Zhengyang Geng and Mingyang Deng and Xingjian Bai and J Zico Kolter and Kaiming He},
booktitle={The Thirty-ninth Annual Conference on Neural Information Processing Systems},
year={2025},
url={https://openreview.net/forum?id=uWj4s7rMnR}
}

@inproceedings{geng2026improved,
  title={Improved mean flows: On the challenges of fastforward generative models},
  author={Geng, Zhengyang and Lu, Yiyang and Wu, Zongze and Shechtman, Eli and Kolter, J Zico and He, Kaiming},
  booktitle={Proceedings of the IEEE/CVF Conference on Computer Vision and Pattern Recognition},
  pages={30467--30476},
  year={2026}
}

@misc{zhang2026teacherfeaturedriftingonestepdiffusion,
      title={Teacher-Feature Drifting: One-Step Diffusion Distillation with Pretrained Diffusion Representations}, 
      author={Yuan Zhang and Chenyi Li and Haodong Yu and Guoqing Ma and Jiajun Zha and Yuanming Yang and Bo Wang and Wei Tang and Wenbo Li and Haoyang Huang and Nan Duan},
      year={2026},
      eprint={2605.07327},
      archivePrefix={arXiv},
      primaryClass={cs.CV},
      url={https://arxiv.org/abs/2605.07327}, 
}

@inproceedings{mokady2022nulltextinversioneditingreal,
  title={Null-text inversion for editing real images using guided diffusion models},
  author={Mokady, Ron and Hertz, Amir and Aberman, Kfir and Pritch, Yael and Cohen-Or, Daniel},
  booktitle={2023 IEEE/CVF Conference on Computer Vision and Pattern Recognition (CVPR)},
  pages={6038--6047},
  year={2023},
  organization={IEEE}
}

@inproceedings{
meng2022sdedit,
title={{SDE}dit: Guided Image Synthesis and Editing with Stochastic Differential Equations},
author={Chenlin Meng and Yutong He and Yang Song and Jiaming Song and Jiajun Wu and Jun-Yan Zhu and Stefano Ermon},
booktitle={International Conference on Learning Representations},
year={2022},
url={https://openreview.net/forum?id=aBsCjcPu_tE}
}

@article{
luzi2024boomerang,
title={Boomerang: Local sampling on image manifolds using diffusion models},
author={Lorenzo Luzi and Paul M Mayer and Josue Casco-Rodriguez and Ali Siahkoohi and Richard Baraniuk},
journal={Transactions on Machine Learning Research},
issn={2835-8856},
year={2024},
url={https://openreview.net/forum?id=NYdThkjNW1},
note={}
}

@article{yu2015lsun,
  title={LSUN: Construction of a Large-scale Image Dataset using Deep Learning with Humans in the Loop},
  author={Yu, Fisher and Seff, Ari and Zhang, Yinda and Song, Shuran and Funkhouser, Thomas and Xiao, Jianxiong},
  journal={arXiv preprint arXiv:1506.03365},
  year={2015}
}

@inproceedings{
chen2024exploring,
title={Exploring Low-Dimensional Subspace in Diffusion Models for Controllable Image Editing},
author={Siyi Chen and Huijie Zhang and Minzhe Guo and Yifu Lu and Peng Wang and Qing Qu},
booktitle={The Thirty-eighth Annual Conference on Neural Information Processing Systems},
year={2024},
url={https://openreview.net/forum?id=50aOEfb2km}
}

@inproceedings{zhang2018unreasonable,
  title     = {The Unreasonable Effectiveness of Deep Features as a Perceptual Metric},
  author    = {Zhang, Richard and Isola, Phillip and Efros, Alexei A. and Shechtman, Eli and Wang, Oliver},
  booktitle = {Proceedings of the IEEE Conference on Computer Vision and Pattern Recognition},
  pages     = {586--595},
  year      = {2018}
}

@book{golub2013matrix,
  title={Matrix computations},
  author={Golub, Gene H and Van Loan, Charles F},
  volume={4},
  year={2013},
  publisher={JHU press}
}

@article{baydin2018automatic,
  title={Automatic differentiation in machine learning: a survey},
  author={Baydin, Atilim Gunes and Pearlmutter, Barak A and Radul, Alexey Andreyevich and Siskind, Jeffrey Mark},
  journal={Journal of machine learning research},
  volume={18},
  number={153},
  pages={1--43},
  year={2018}
}

@article{
oquab2024dinov,
title={{DINO}v2: Learning Robust Visual Features without Supervision},
author={Maxime Oquab and Timoth{\'e}e Darcet and Th{\'e}o Moutakanni and Huy V. Vo and Marc Szafraniec and Vasil Khalidov and Pierre Fernandez and Daniel HAZIZA and Francisco Massa and Alaaeldin El-Nouby and Mido Assran and Nicolas Ballas and Wojciech Galuba and Russell Howes and Po-Yao Huang and Shang-Wen Li and Ishan Misra and Michael Rabbat and Vasu Sharma and Gabriel Synnaeve and Hu Xu and Herve Jegou and Julien Mairal and Patrick Labatut and Armand Joulin and Piotr Bojanowski},
journal={Transactions on Machine Learning Research},
issn={2835-8856},
year={2024},
url={https://openreview.net/forum?id=a68SUt6zFt},
note={}
}
\bibliographystyle{iclr2027_conference}

\clearpage
\appendix
\raggedbottom
\section{Appendix}




\subsection{Ring of Gaussians toy experiment}
\label{app:repulsion-limits}

\paragraph{Toy distribution and velocity.}
Figure~\ref{fig:oscar_steps} starts from a centered two-dimensional Gaussian
with standard deviation 0.25. The target is a mixture of sixteen equally likely
Gaussian components whose centers are evenly spaced on a circle of radius 2.
Each component has standard deviation 0.08. Thus, each black plus sign in the
figure is the center of one target Gaussian. For every source-target pair, we
use the straight path $X_t=(1-t)Z+tY$. The velocity is
$v_t(x)=\mathbb E[Y-Z\mid X_t=x]$, which is the average displacement of all pairs whose
path could pass through $x$ at time $t$. Near the center, several target modes
are plausible and contribute to this average. As a trajectory approaches the
ring, nearby modes receive more weight and determine its final direction. The
distribution is analytic, so we evaluate this velocity directly without
training a network. Specifically, with
$c_j=2(\cos(2\pi j/16),\sin(2\pi j/16))^\top$, the exact field is
\[
\begin{aligned}
q_t&=(1-t)^2(0.25)^2+t^2(0.08)^2,
&a_t&=\frac{t(0.08)^2-(1-t)(0.25)^2}{q_t},\\
w_j(x,t)&=
\frac{\exp\!\left(-\|x-tc_j\|_2^2/(2q_t)\right)}
{\sum_{\ell=0}^{15}\exp\!\left(-\|x-tc_\ell\|_2^2/(2q_t)\right)},
&v_t(x)&=a_tx+(1-a_tt)\sum_{j=0}^{15}w_j(x,t)c_j.
\end{aligned}
\]
Here, $w_j(x,t)$ is the weight assigned to target mode $j$ at the current
location and time.


\paragraph{Compared methods.}
All three panels start from the same eight source samples and use 800
integration steps. The left panel uses ordinary flow matching. The middle
panel adds deterministic log-determinant repulsion at every step, with a force
that gradually vanishes near the endpoint. This isolates the deterministic
repulsive component used to motivate OSCAR \citep{wu2025oscar}. The right panel applies JIVE only once, after step 100.
For each sample separately, JIVE finds the local direction that most strongly
changes its predicted endpoint, moves the sample by 0.25 in one of the two
signs along that direction, and then resumes ordinary flow matching. JIVE does
not inspect or coordinate the other seven samples.

\subsection{Proofs of the main results}
\label{app:theorem-proofs}

\begin{proof}[Proof of Theorem~\ref{thm:pairwise-diversity}]
Conditional uniformity on the unit sphere gives
\[
\mathbb E[\delta\mid X]=0,\qquad
\mathbb E[\delta\delta^\top\mid X]=\frac{\varepsilon^2}{k}Q(X)Q(X)^\top.
\]
Write $B=J(X)\delta$. Then $\mathbb E[B\mid X]=0$, so the cross-covariance of $Y=F_t(X)$ and $B$ vanishes. Hence by the law of total covariance
\[
\operatorname{Cov}(\widetilde Y)
=\operatorname{Cov}(Y)+\frac{\varepsilon^2}{k}
\mathbb E[J(X)Q(X)Q(X)^\top J(X)^\top].
\]
For independent identically distributed vectors $W_1,W_2$ with finite second moments, expanding the square yields
$\mathbb E\|W_1-W_2\|_2^2=2\operatorname{Tr}\operatorname{Cov}(W_1)$.
Applying this identity to $\widetilde Y$ and using cyclicity of the trace proves the formula.
\end{proof}

\begin{proposition}[Optimal trace subspace]
\label{prop:optimal-trace-subspace}
For $J_t\in\mathbb R^{d\times d}$ with singular values $\sigma_1\geq\cdots\geq\sigma_d\geq0$ and any $Q\in\mathbb R^{d\times k}$ with $Q^\top Q=I_k$,
\[
\operatorname{Tr}(Q^\top J_t^\top J_tQ)\leq\sum_{i=1}^k\sigma_i^2,
\]
with equality for $Q=[v_1,\ldots,v_k]$.
\end{proposition}

\begin{proof}[Proof of Proposition~\ref{prop:optimal-trace-subspace}]
Let $P=QQ^\top$ and write $J_t^\top J_t=\sum_i\sigma_i^2v_iv_i^\top$. Then
\[
\operatorname{Tr}(Q^\top J_t^\top J_tQ)=\sum_i\sigma_i^2w_i,
\qquad w_i=v_i^\top Pv_i\in[0,1],\quad\sum_iw_i=k.
\]
The largest weighted sum assigns unit weight to the $k$ largest squared singular values. Choosing $Q=[v_1,\ldots,v_k]$ achieves this bound.
\end{proof}


\subsection{Direction estimation algorithm}
\label{app:direction-details}

Algorithm~\ref{alg:jacobian-rf-sampling} is the concrete rank-one implementation used for rectified-flow generation. 

\begin{algorithm}[H]
\caption{Jacobian-guided trajectory perturbation for rectified flow}
\label{alg:jacobian-rf-sampling}
\begin{algorithmic}[1]
\Require Velocity field $v_\theta(x,t)$, total denoising steps $N$, initial noise $x_0\sim\mathcal N(0,\mathbf I)$, power iterations $K_{\mathrm{pow}}$, finite-difference step $\epsilon$, target perturbation norm $\gamma$
\Ensure Generated sample $x_N$
\State Sample $v^{(0)}\sim\mathcal N(0,\mathbf I)$
\For{$\ell=1\ \text{to}\ K_{\mathrm{pow}}$} \Comment{Power iteration via finite differences}
  \State $v^{(\ell)}\gets\frac{[x_0+\epsilon v^{(\ell-1)}+v_\theta(x_0+\epsilon v^{(\ell-1)},0)]-[x_0+v_\theta(x_0,0)]}{\epsilon}$
  \State $v^{(\ell)}\gets v^{(\ell)}/\|v^{(\ell)}\|_2$
\EndFor
\State Sample $s \sim \text{Uniform}(\{-1, +1\})$
\State $x_0\gets x_0+s\gamma\,v^{(K_{\mathrm{pow}})}$ \Comment{Add perturbation vector of norm $\gamma$}
\For{$i=0\ \text{to}\ N-1$}
  \State $t_i\gets i/N$, \quad $\Delta t\gets 1/N$
  \State $x_{i+1}\gets x_i+\Delta t\,v_\theta(x_i,t_i)$
\EndFor
\State \Return $x_N$
\end{algorithmic}
\end{algorithm}

\subsection{Eigenvector and singular vector alignment}
\label{app:gaussian-model}

In this section, we present the observation that the top eigenvector and top singular vector have high alignment in our tested models. Alignment is especially high for multi-step models.

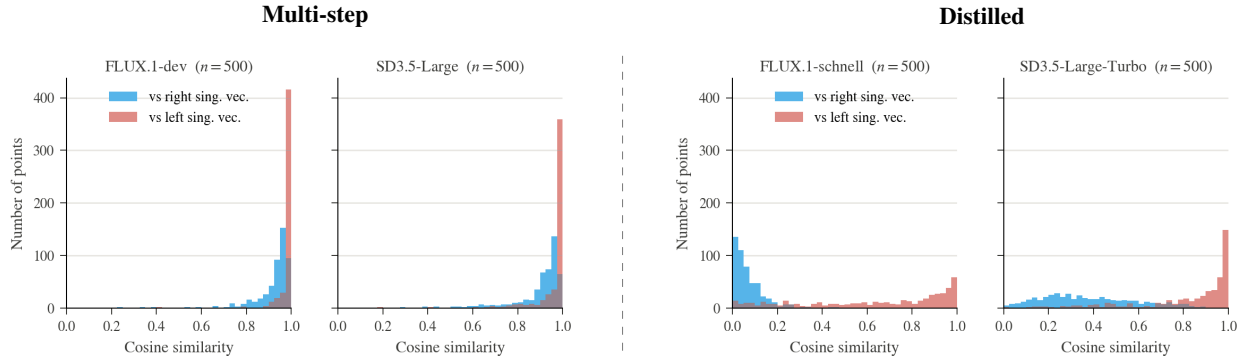
\begin{figure}[H]
\centering
\begin{tikzpicture}
\node[anchor=south west,inner sep=0] (plot) at (0,0)
  {\includegraphics[width=\linewidth]{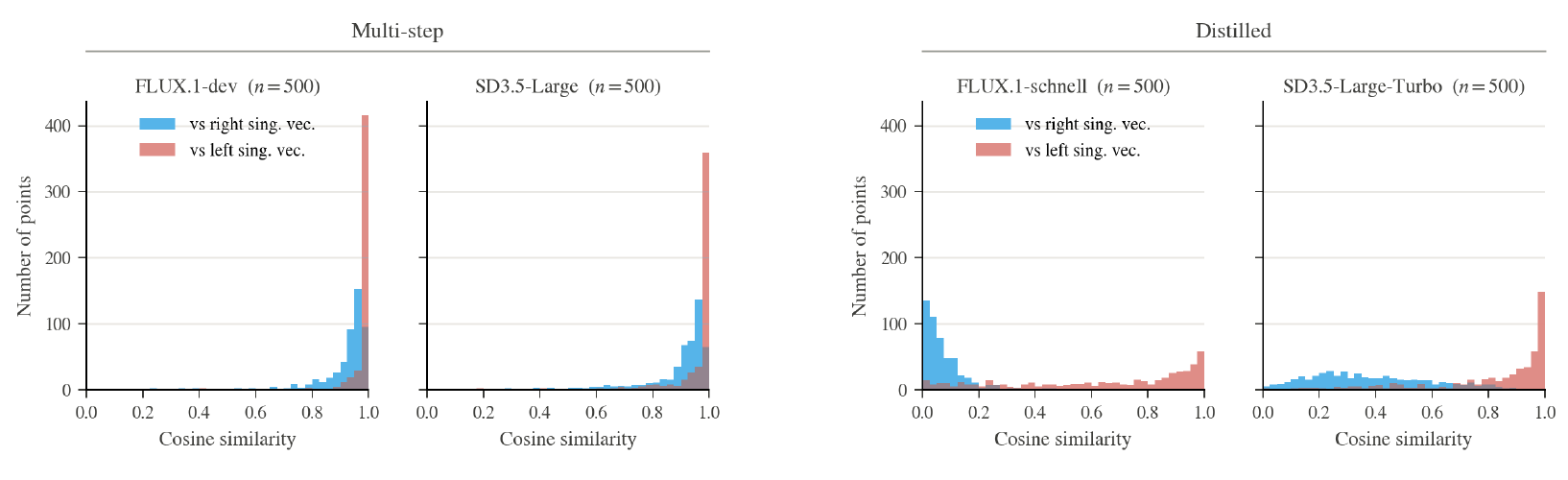}};
\begin{scope}[x={(plot.south east)},y={(plot.north west)}]
\fill[white] (0,0.875) rectangle (1,1);
\node[text=black,font=\fontsize{9.02}{11}\selectfont\bfseries] at (0.254,0.953) {Multi-step};
\node[text=black,font=\fontsize{9.02}{11}\selectfont\bfseries] at (0.788,0.953) {Distilled};
\draw[gray,dashed] (0.5,0.08) -- (0.5,0.87);
\end{scope}
\end{tikzpicture}
\caption{\textbf{Eigenvector and singular vector alignment across models.} The histograms show absolute cosine similarity of the leading eigenvector with the leading right and left singular vectors at $t=0$ ($n=500$ per model). Blue/red denote right/left singular vectors throughout. Alignment is strong in multi-step models (left) and weaker in distilled models (right), especially with the right singular vector.}
\label{fig:eig-sv-alignment}
\end{figure}

\subsection{Why random directions don't work.}
\label{app:random_dir}

We observed that random perturbation directions in general fail to meaningfully alter the final image. This can be explained by the endpoint displacement approximation $\|F_t(x_t+\delta)-F_t(x_t)\|_2^2\approx \|J_t\delta\|_2^2$. We have that 
\[
\|J_t\delta\|_2^2=\sum_i\sigma_i^2\langle \delta,v_i\rangle^2,
\]

so for a random direction $\delta$ much of its energy is placed in right singular vectors $v_i$ with small singular values. Thus, $\|J_t\delta\|_2^2$ as a whole is small. We verify in Figure \ref{fig:spectrum_decay} that the spectrum in distilled models decays rapidly, so aligning with top singular directions results in meaningfully higher $\|J_t\delta\|_2^2$. We compare tradeoff curves of one $J_t$ iteration with random directions in Figure \ref{fig:jvp-vs-random}.

\begin{figure}[!htb]
\captionsetup{justification=raggedright,singlelinecheck=false}
\centering
\includegraphics[width=0.5\linewidth]{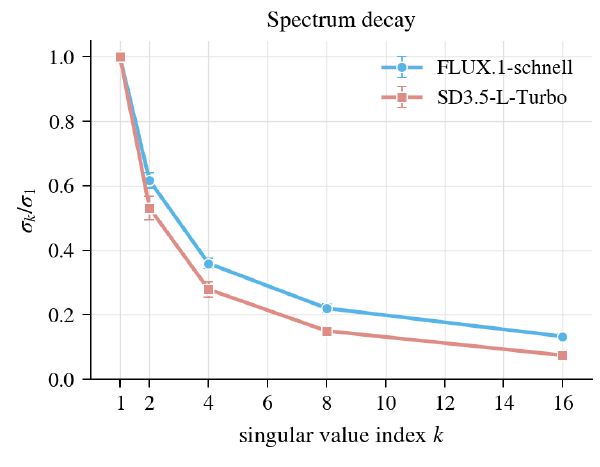}
\caption{\textbf{The leading singular values decay rapidly.} Spectra for two distilled models, normalized by the largest singular value. The spectrum decays rapidly in both models, demonstrating that alignment with top directions allows for more endpoint displacement per unit perturbation}
\label{fig:spectrum_decay}
\end{figure}

\subsection{Runtime and Approximation Power for Perturbation Computation Methods}
\label{app:runtime-comparison}

We compare the approximation power of the two computational methods in Section \ref{subsec:compute_direction}: $J_t^\top J_t$ and JVP. The approximation power is demonstrated in Table~\ref{tab:ek_gain_full} and the full table of run time is in Table \ref{tab:runtime_full}.

Based on Table \ref{tab:ek_gain_full}, $J_t^\top J_t$ is our stronger low-cost estimator. One product places more than 91\% of the direction's energy in the top four right singular directions in both models. 
Although the JVP direction may not align with a top singular
vector, Table~\ref{tab:ek_gain_full} shows that it captures substantially more
energy in the leading singular subspace than a random direction and
achieves greater endpoint diversity gain. This is our cheapest alternative.

\begin{table}[htbp]
\centering
\caption{
Energy $E_k=\sum_{i=1}^k|\langle \delta,v_i\rangle|^2$ in the top-$k$
right singular subspace for unit $\delta$, and gain relative to
$\sigma_{\mathrm{rms}}=(\sum_i \sigma_i^2/d)^{1/2}$.
The random row is an analytic value ($k/d$).
Both JVP and $J_t^\top J_t$ are evaluated with one iteration.
}
\label{tab:ek_gain_full}

\small
\setlength{\tabcolsep}{4pt}

\begin{tabular}{llcccccc}
\toprule
Model
& Direction
& $E_1$
& $E_2$
& $E_4$
& $E_8$
& $E_{16}$
& $\|Jw\|_2/\sigma_{\mathrm{rms}}$
\\
\midrule

\multirow{3}{*}{FLUX.1-schnell}
& Random
& 1.5e{-}5
& 3.1e{-}5
& 6.1e{-}5
& 1.2e{-}4
& 2.4e{-}4
& 1.0
\\

& \textcolor{paperBlue}{JVP}
& \textcolor{paperBlue}{0.033}
& \textcolor{paperBlue}{0.050}
& \textcolor{paperBlue}{0.072}
& \textcolor{paperBlue}{0.102}
& \textcolor{paperBlue}{0.127}
& \textcolor{paperBlue}{30.8}
\\

& \textcolor{paperOrange}{$J_t^\top J_t$}
& \textcolor{paperOrange}{0.624}
& \textcolor{paperOrange}{0.803}
& \textcolor{paperOrange}{0.915}
& \textcolor{paperOrange}{0.969}
& \textcolor{paperOrange}{0.989}
& \textcolor{paperOrange}{198.3}
\\

\midrule

\multirow{3}{*}{SD3.5-Large-Turbo}
& Random
& 3.8e{-}6
& 7.6e{-}6
& 1.5e{-}5
& 3.1e{-}5
& 6.1e{-}5
& 1.0
\\

& \textcolor{paperBlue}{JVP}
& \textcolor{paperBlue}{0.074}
& \textcolor{paperBlue}{0.108}
& \textcolor{paperBlue}{0.152}
& \textcolor{paperBlue}{0.226}
& \textcolor{paperBlue}{0.277}
& \textcolor{paperBlue}{102.7}
\\

& \textcolor{paperOrange}{$J_t^\top J_t$}
& \textcolor{paperOrange}{0.648}
& \textcolor{paperOrange}{0.816}
& \textcolor{paperOrange}{0.916}
& \textcolor{paperOrange}{0.981}
& \textcolor{paperOrange}{0.992}
& \textcolor{paperOrange}{321.6}
\\

\bottomrule
\end{tabular}
\end{table}

\begin{table}[H]
\centering
\caption{Runtime and memory overhead of Jacobian-vector and \(J_t^\top J_t\) products on FLUX.1-schnell at \(512\times512\) resolution on an A100. Only the transformer is resident on the GPU. FD is finite difference.}
\label{tab:runtime_full}
\small
\setlength{\tabcolsep}{5pt}
\begin{tabular}{lcccc}
\toprule
Configuration & Seconds & Time$\times$Base & FLOPs$\times$Base & Peak Mem.\,(GB) \\
\midrule
Baseline & 0.307 & 1.00 & 1.00 & 22.3 \\
JVP FD (forward) & 0.617 & 2.01 & 2.00 & 22.3 \\
JVP FD (central) & 0.619 & 2.02 & 2.00 & 22.3 \\
JVP exact & 1.733 & 5.64 & 3.00 & 23.5 \\
$J_t^\top J_t$ FD (forward) & 1.302 & 4.24 & 5.04 & 35.4 \\
$J_t^\top J_t$ exact & 2.425 & 7.90 & 6.03 & 35.4 \\
$J_t^\top J_t$ FD (forward) w/ grad checkpointing & 1.651 & 5.38 & 6.04 & 23.0 \\
$J_t^\top J_t$ exact w/ grad checkpointing & 2.790 & 9.09 & 7.03 & 23.6 \\
\bottomrule
\end{tabular}
\end{table}

\begin{table}[htpb]
\centering
\caption{Comparing the low-cost version of JIVE with using extra denoising steps. Experiments were conducted using GenEval prompts at 4 images per prompt.}
\newcommand{\partise}[2]{#1{\fontsize{6}{7}\selectfont$\pm$#2}}
\newcommand{\jtablecell}[1]{\textcolor{paperBlue}{#1}}
\newcommand{\jtjtablecell}[1]{\textcolor{paperOrange}{#1}}
\setlength{\tabcolsep}{5pt}
\begin{tabular}{lrrrrrc}
\toprule
& \multicolumn{2}{c}{Diversity}
& \multicolumn{4}{c}{Fidelity} \\
\cmidrule(lr){2-3}\cmidrule(lr){4-7}
Method & F-Vendi $\uparrow$ & P-Vendi $\uparrow$ & CLIP $\uparrow$ & HPSv2 $\uparrow$ & PickScore $\uparrow$ & GenEval $\uparrow$ \\
\midrule
\multicolumn{7}{l}{\textbf{FLUX.1-schnell}
\quad (1 step, $512\times512$)} \\
\midrule
Deterministic
& \partise{1.940}{0.019} & \partise{1.683}{0.013} & \partise{33.189}{0.138} & \partise{30.540}{0.109} & \partise{23.330}{0.044} & 0.700 \\
Extra denoising steps (5)
& \partise{2.028}{0.021} & \partise{1.710}{0.013} & \partise{32.934}{0.134} & \partise{30.483}{0.109} & \partise{23.249}{0.043} & 0.686 \\
\jtablecell{\textbf{Ours} (JVP, FD)}
& \jtablecell{\partise{2.200}{0.021}} & \jtablecell{\partise{1.964}{0.018}} & \jtablecell{\partise{32.932}{0.132}} & \jtablecell{\partise{28.890}{0.119}} & \jtablecell{\partise{22.945}{0.044}} & \jtablecell{0.684} \\
\jtjtablecell{\textbf{Ours} ($J_t^\top J_t$, FD)}
& \jtjtablecell{\textbf{\partise{2.272}{0.023}}} & \jtjtablecell{\textbf{\partise{2.080}{0.016}}} & \jtjtablecell{\partise{32.909}{0.132}} & \jtjtablecell{\partise{28.858}{0.115}} & \jtjtablecell{\partise{22.910}{0.043}} & \jtjtablecell{0.672} \\
\midrule
\multicolumn{7}{l}{\textbf{SD3.5-Large-Turbo}
\quad (4 steps, $1024\times1024$)} \\
\midrule
Deterministic
& \partise{1.932}{0.021} & \partise{1.369}{0.007} & \partise{33.650}{0.127} & \partise{28.751}{0.116} & \partise{23.112}{0.043} & 0.708 \\
Extra denoising steps (8)
& \partise{1.998}{0.022} & \partise{1.485}{0.010} & \partise{33.434}{0.130} & \partise{29.465}{0.119} & \partise{23.230}{0.043} & 0.694 \\
\jtablecell{\textbf{Ours} (JVP, FD)}
& \jtablecell{\partise{2.195}{0.023}} & \jtablecell{\partise{1.964}{0.015}} & \jtablecell{\partise{32.860}{0.125}} & \jtablecell{\partise{27.871}{0.119}} & \jtablecell{\partise{22.677}{0.043}} & \jtablecell{0.657} \\
\jtjtablecell{\textbf{Ours} ($J_t^\top J_t$, FD)}
& \jtjtablecell{\textbf{\partise{2.276}{0.023}}} & \jtjtablecell{\textbf{\partise{1.977}{0.013}}} & \jtjtablecell{\partise{32.910}{0.126}} & \jtjtablecell{\partise{27.667}{0.121}} & \jtjtablecell{\partise{22.638}{0.041}} & \jtjtablecell{0.674} \\
\bottomrule
\end{tabular}
\label{tab:compute-matched}
\end{table}


\subsection{Additional tradeoff curves}
\label{app:direction-ablations}

\begin{figure}[H]
    \centering
    \includegraphics[width=0.6\linewidth]{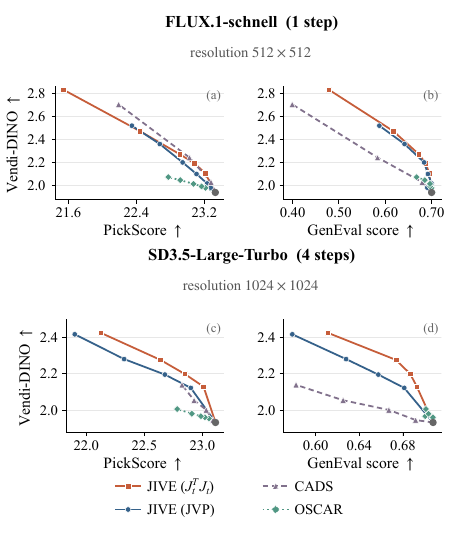}
    \caption{\textbf{PickScore and GenEval score tradeoffs on GenEval.} Companion to Figure~\ref{fig:geneval-frontiers} with the same settings: PickScore (left) and GenEval score (right) versus Vendi-DINO, averaged over 553 prompts. On both FLUX.1-schnell and SD3.5-Large-Turbo, both low-cost variants of JIVE have a better diversity-GenEval tradeoff than CADS or OSCAR.}
    \label{fig:geneval-appendix-frontiers}
\end{figure}

\begin{figure}[H]
    \centering
    \includegraphics[width=1.0\textwidth]{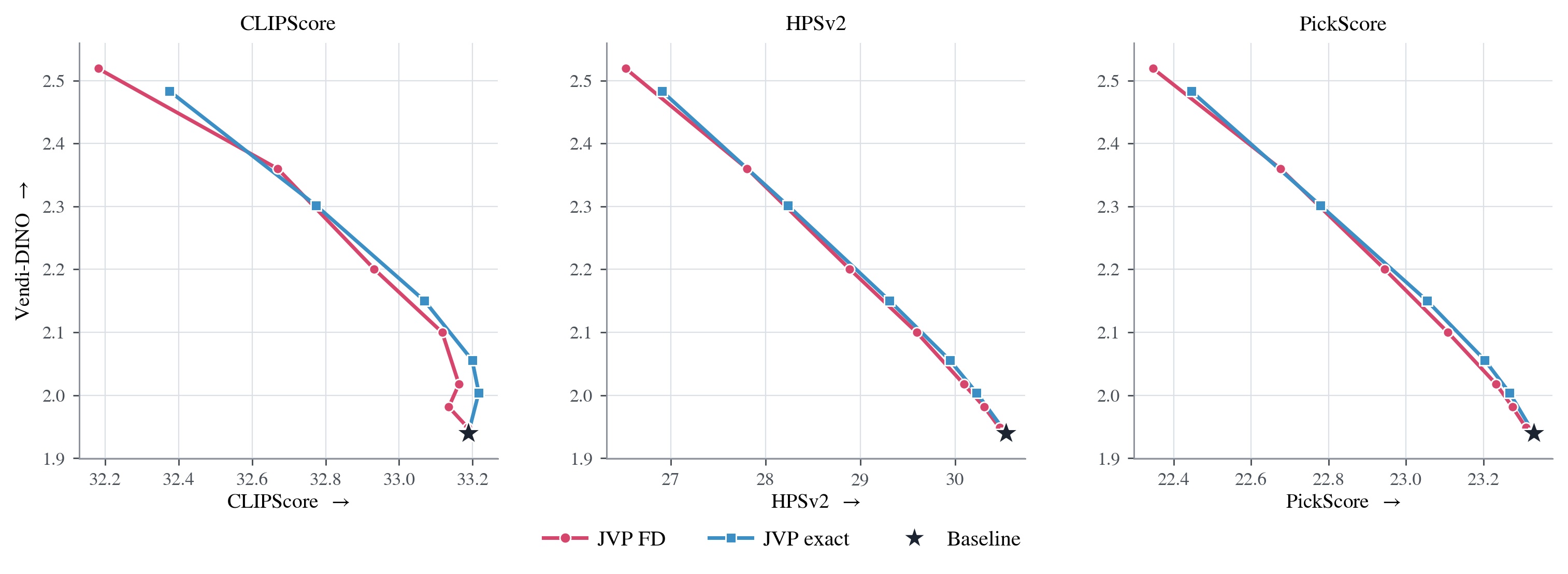}
    \caption{Finite-difference and exact JVP estimates produce similar diversity-quality frontiers. Points are evaluated on Flux.1-schnell at 1 step using GenEval prompts at 4 images per prompt}
    \label{fig:jvp-fd-vs-exact}
\end{figure}

\begin{figure}[H]
    \centering
    \includegraphics[width=1.0\textwidth]{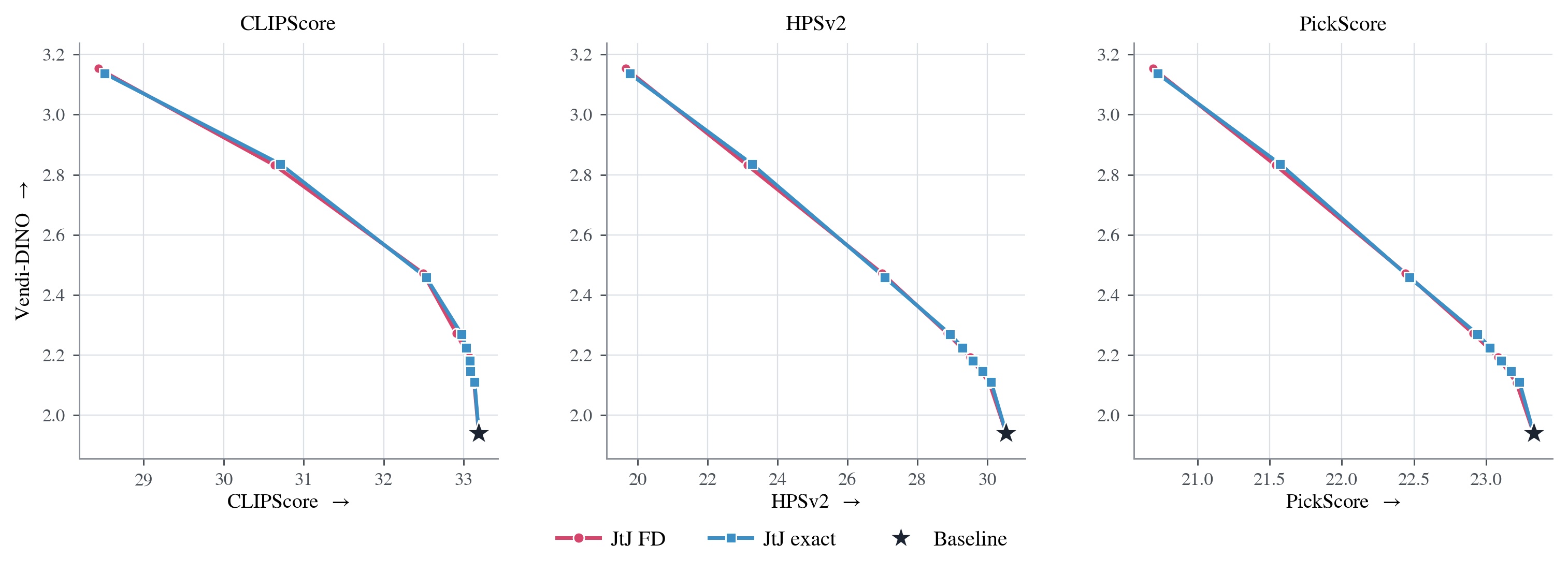}
    \caption{Finite-difference and exact $J^\top J$ estimates produce similar diversity-quality frontiers. Points are evaluated on Flux.1-schnell at 1 step using GenEval prompts at 4 images per prompt}
    \label{fig:jtj-fd-vs-exact}
\end{figure}
\begin{figure}[H]
    \centering
    \includegraphics[width=1.0\textwidth]{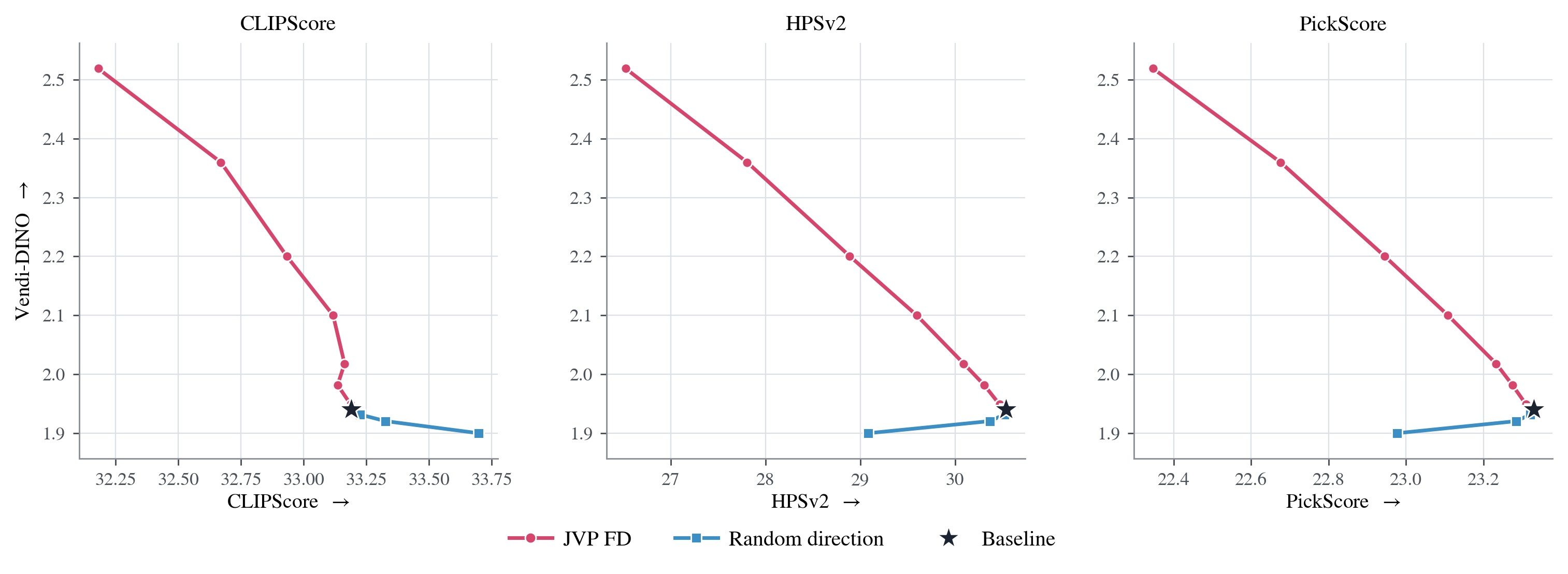}
    \caption{Jacobian-guided perturbation versus random perturbation. Random directions yield almost none of the diversity gain, so the improvement comes from the choice of direction rather than from injected norm. Points are evaluated on Flux.1-schnell at 1 step using GenEval prompts at 4 images per prompt}
    \label{fig:jvp-vs-random}
\end{figure}

\FloatBarrier

\newpage
\subsection{Additional qualitative samples}
\label{app:qualitative-samples}

\begin{figure}[H]
    \centering
    \includegraphics[width=0.80\linewidth]{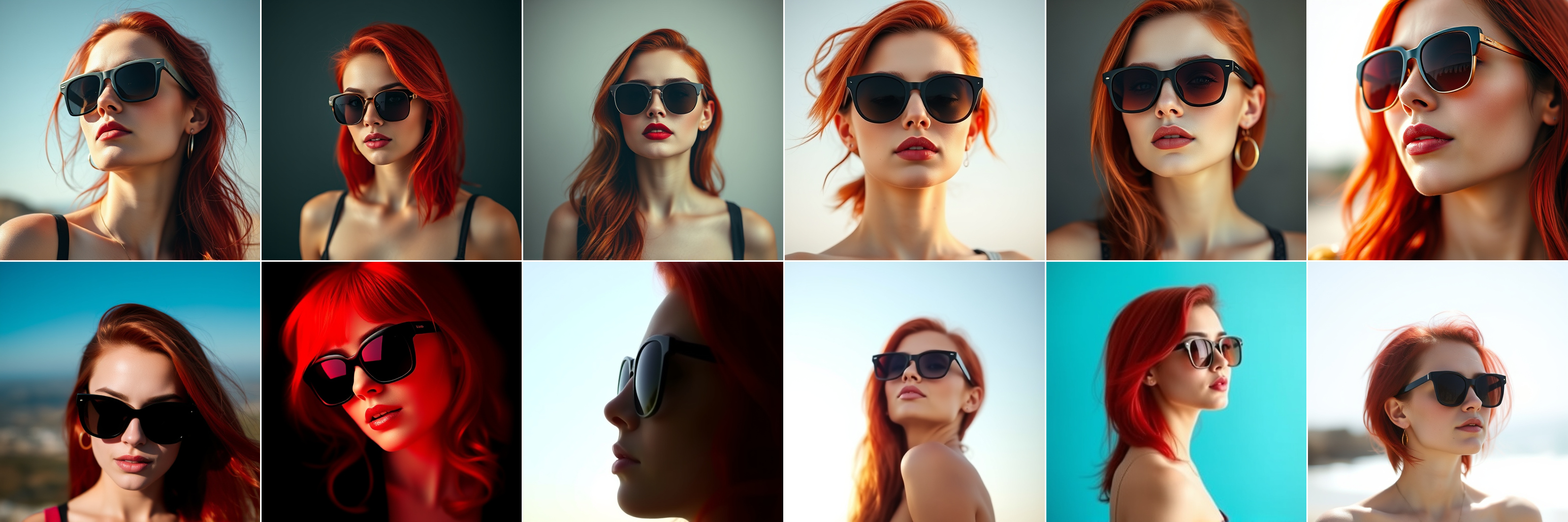}\\[1pt]
    {\small a woman with sunglasses and red hair}\\[5pt]
    \includegraphics[width=0.80\linewidth]{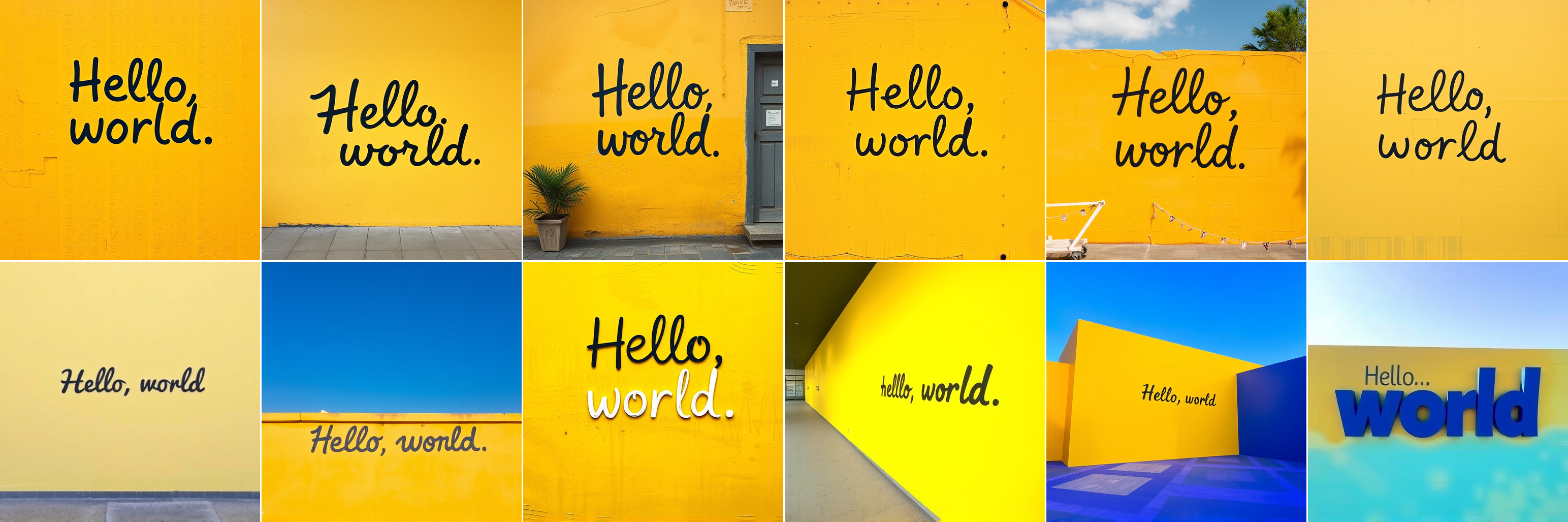}\\[1pt]
    {\small a yellow wall with ``Hello, world.'' written on it}\\[5pt]
    \includegraphics[width=0.80\linewidth]{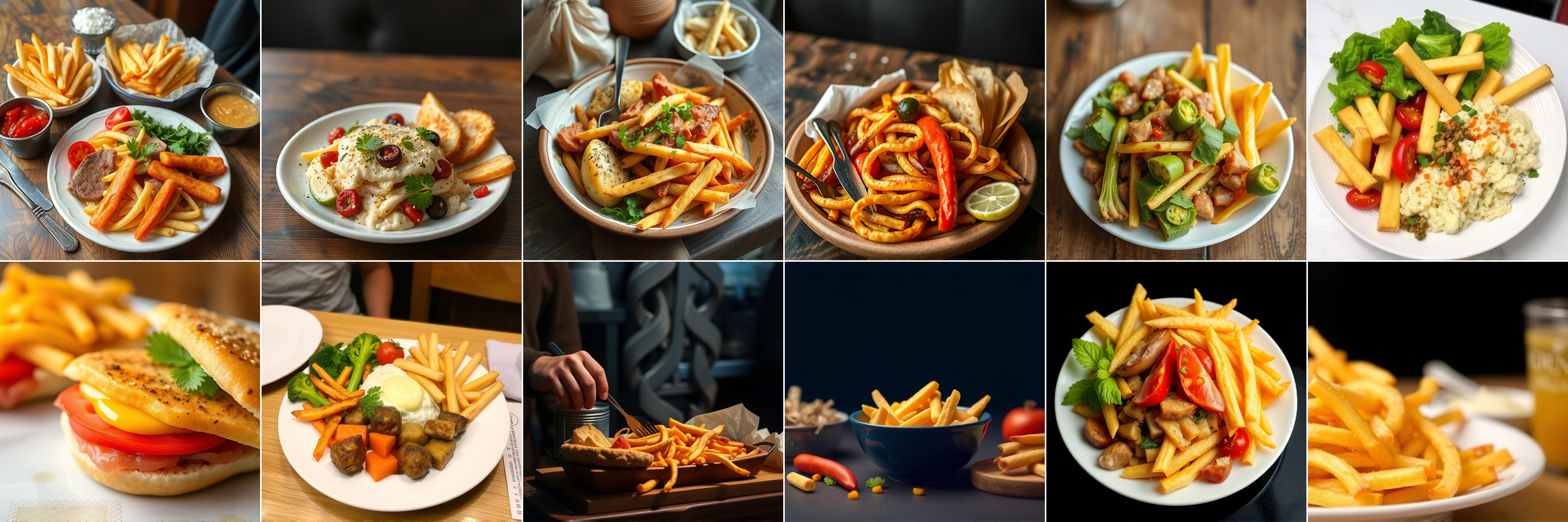}\\[1pt]
    {\small food}\\[5pt]
    \includegraphics[width=0.80\linewidth]{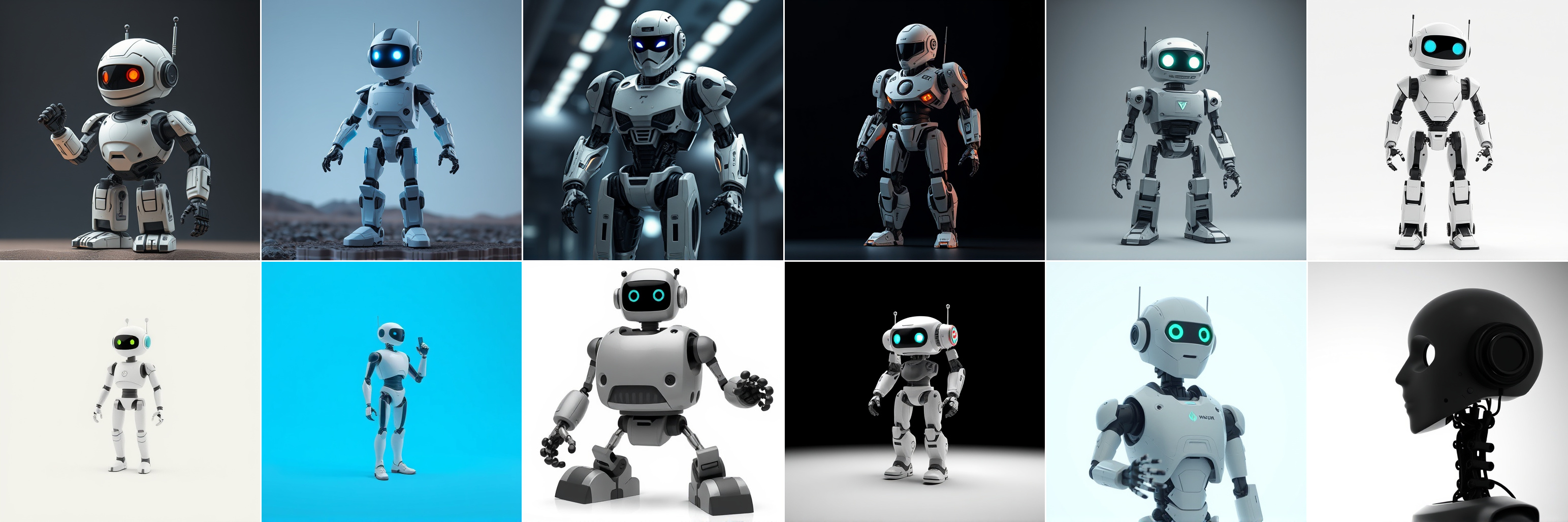}\\[1pt]
    {\small a robot}
    \caption{\textbf{Qualitative diversity across twelve prompts.} Each prompt block shows six matched samples from the unperturbed sampler in the upper row and JIVE in the lower row. Both rows use the same initial seeds. Across varied subjects and compositions, JIVE produces broader changes in appearance, viewpoint, and style while preserving the prompt content. Prompts 1-4 are shown here, and prompts 5-12 continue on the following pages.}
    \label{fig:qualitative-selected-prompts}
\end{figure}

\begin{figure}[p]
    \ContinuedFloat
    \centering
    \includegraphics[width=0.80\linewidth]{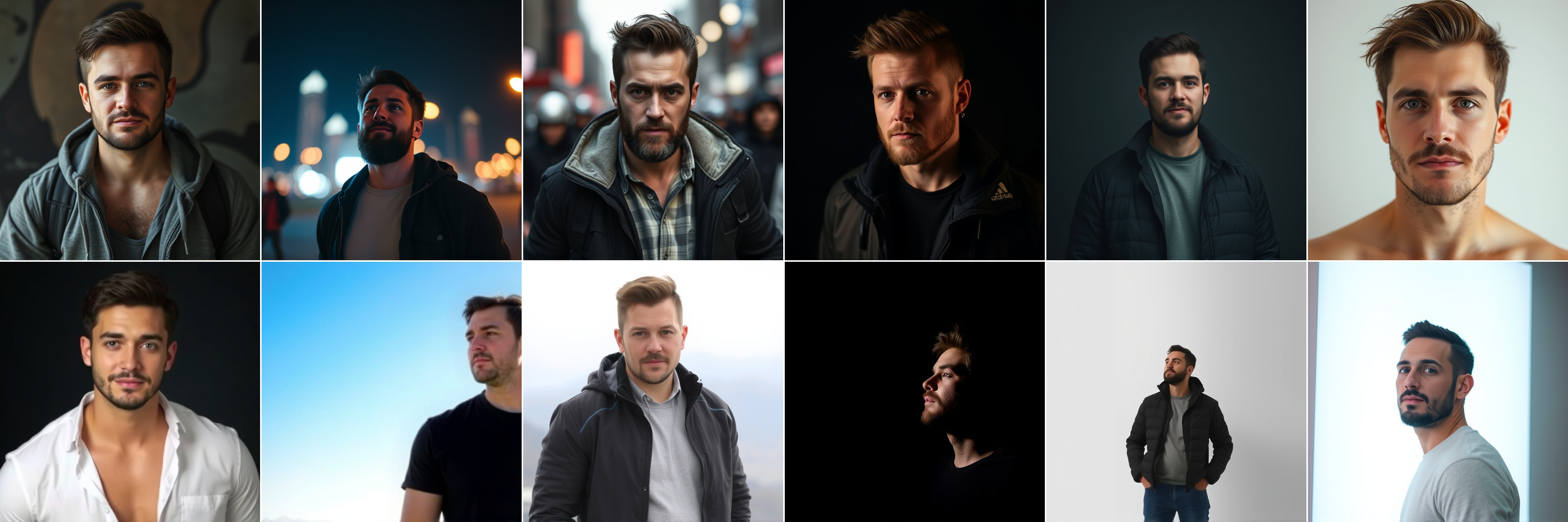}\\[1pt]
    {\small a man}\\[5pt]
    \includegraphics[width=0.80\linewidth]{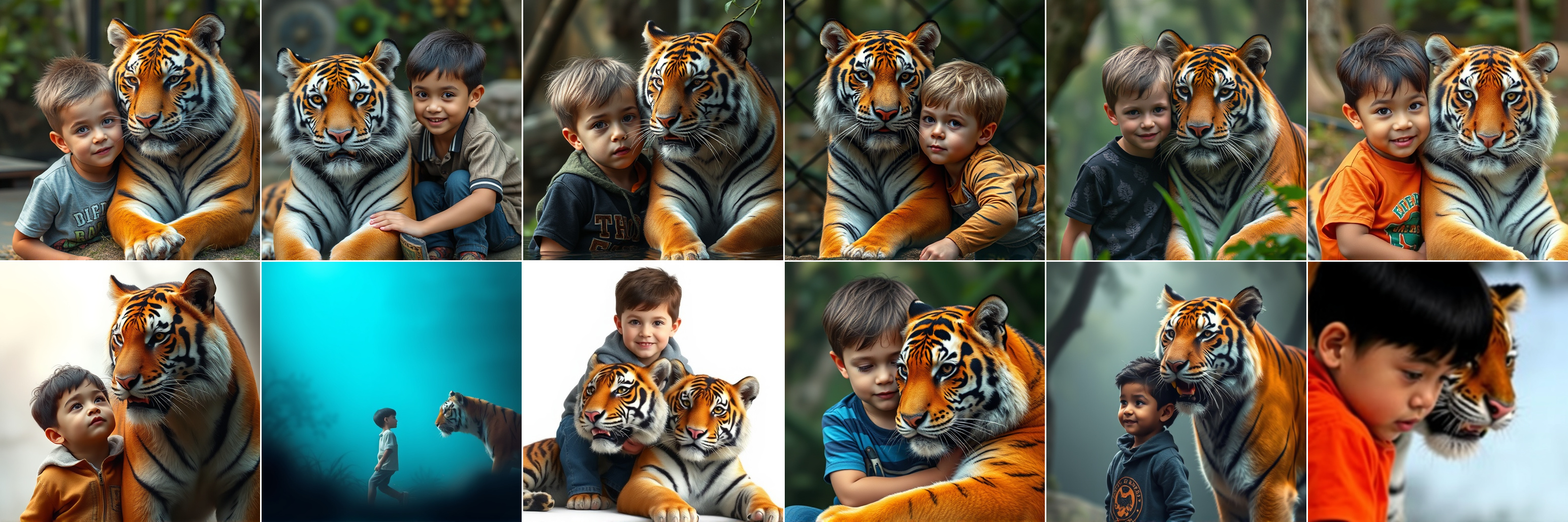}\\[1pt]
    {\small a boy and a tiger}\\[5pt]
    \includegraphics[width=0.80\linewidth]{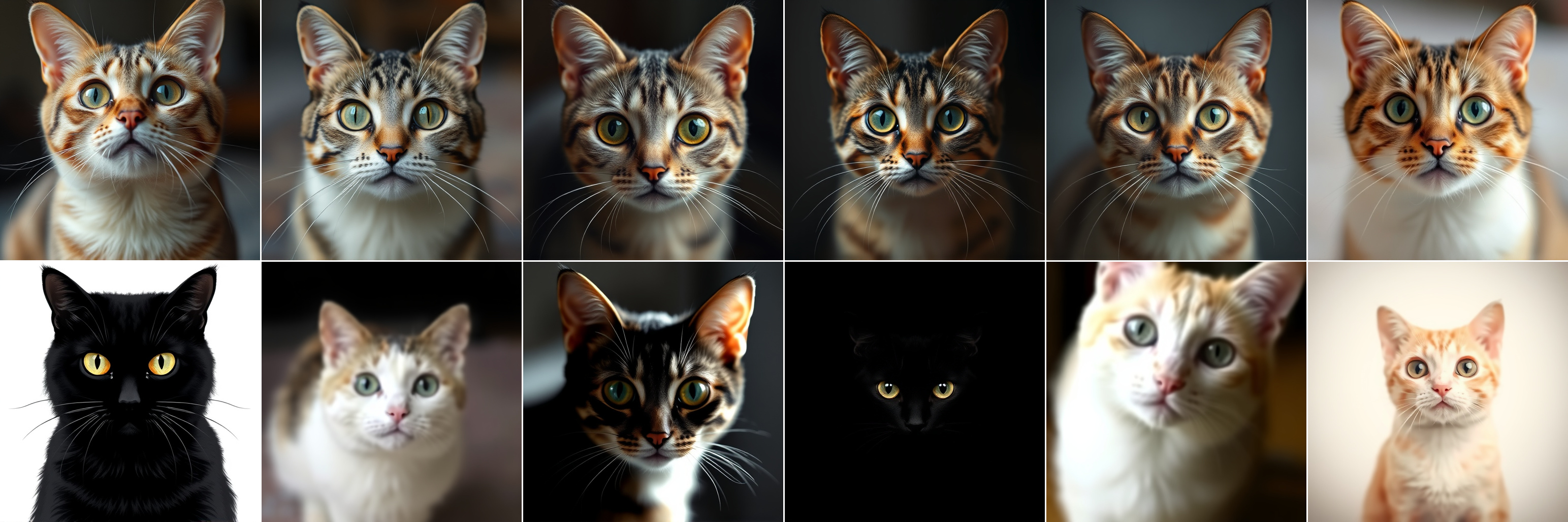}\\[1pt]
    {\small a cat}\\[5pt]
    \includegraphics[width=0.80\linewidth]{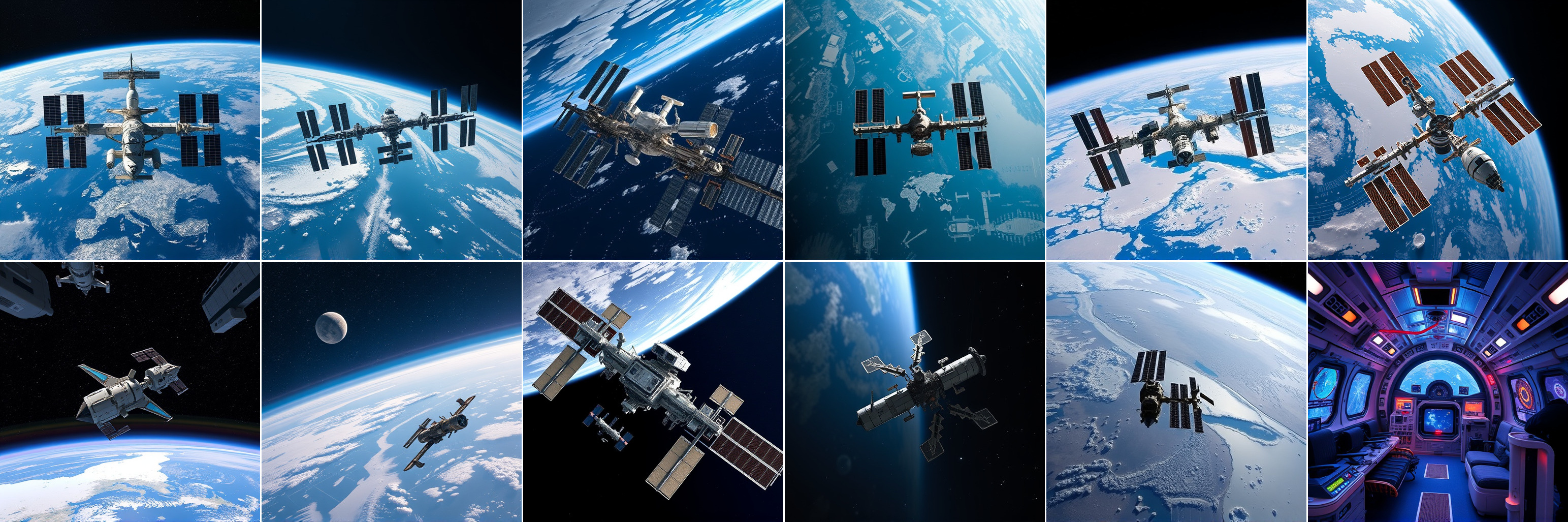}\\[1pt]
    {\small the International Space Station}
    \caption{\textbf{Qualitative diversity across twelve prompts (continued).} Prompts 5-8. Upper rows show unperturbed samples and lower rows show JIVE samples generated from the same initial seeds.}
\end{figure}

\begin{figure}[p]
    \ContinuedFloat
    \centering
    \includegraphics[width=0.80\linewidth]{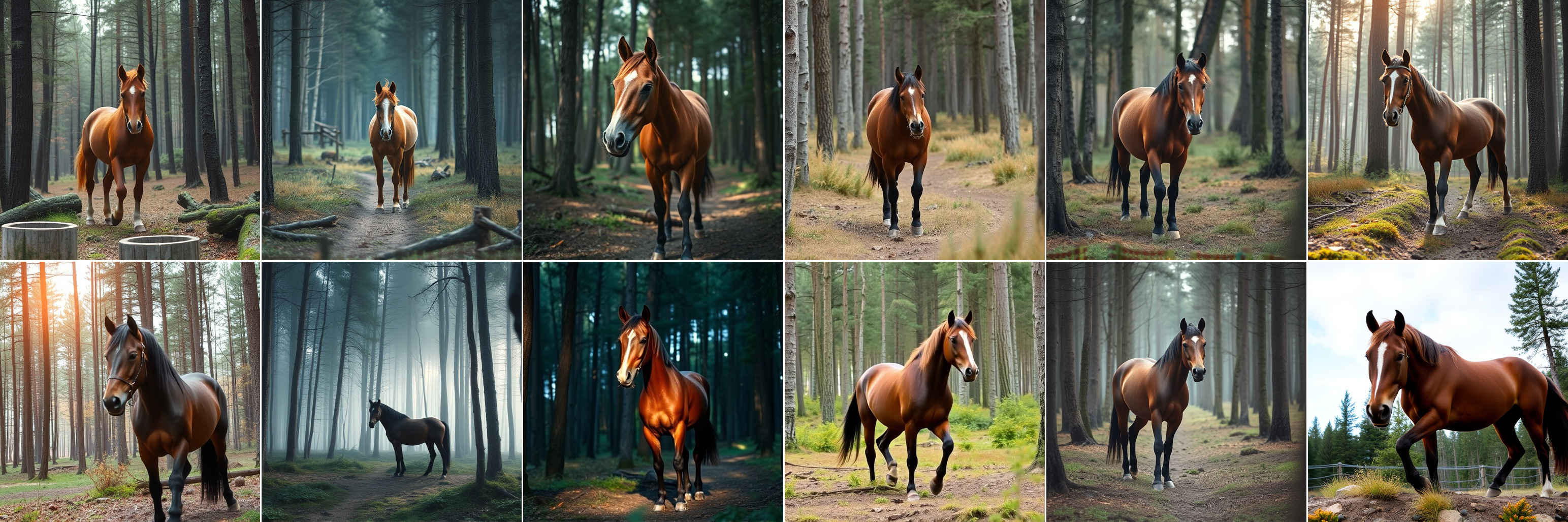}\\[1pt]
    {\small a horse in a forest}\\[5pt]
    \includegraphics[width=0.80\linewidth]{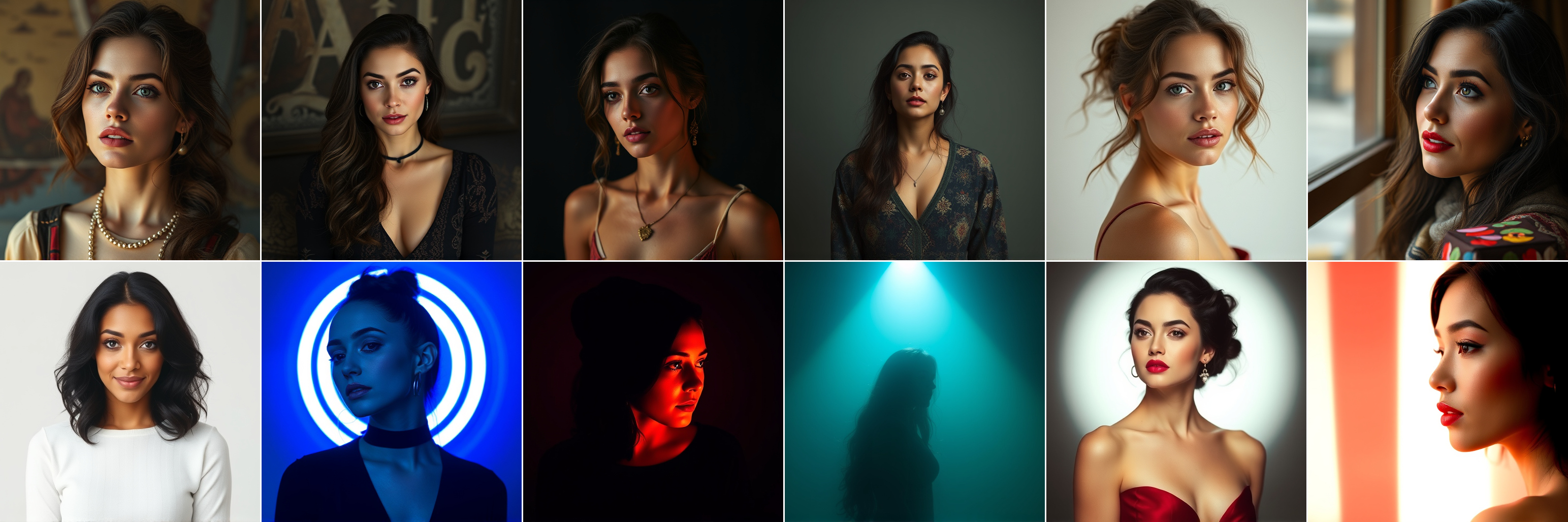}\\[1pt]
    {\small a woman}\\[5pt]
    \includegraphics[width=0.80\linewidth]{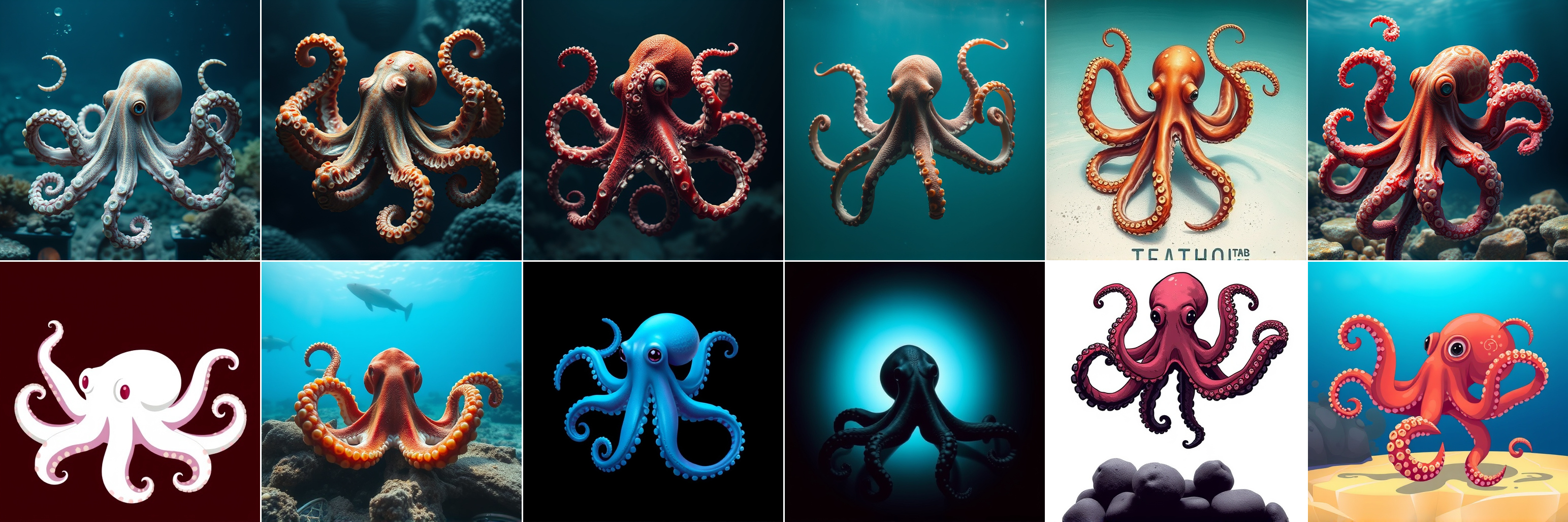}\\[1pt]
    {\small an octopus}\\[5pt]
    \includegraphics[width=0.80\linewidth]{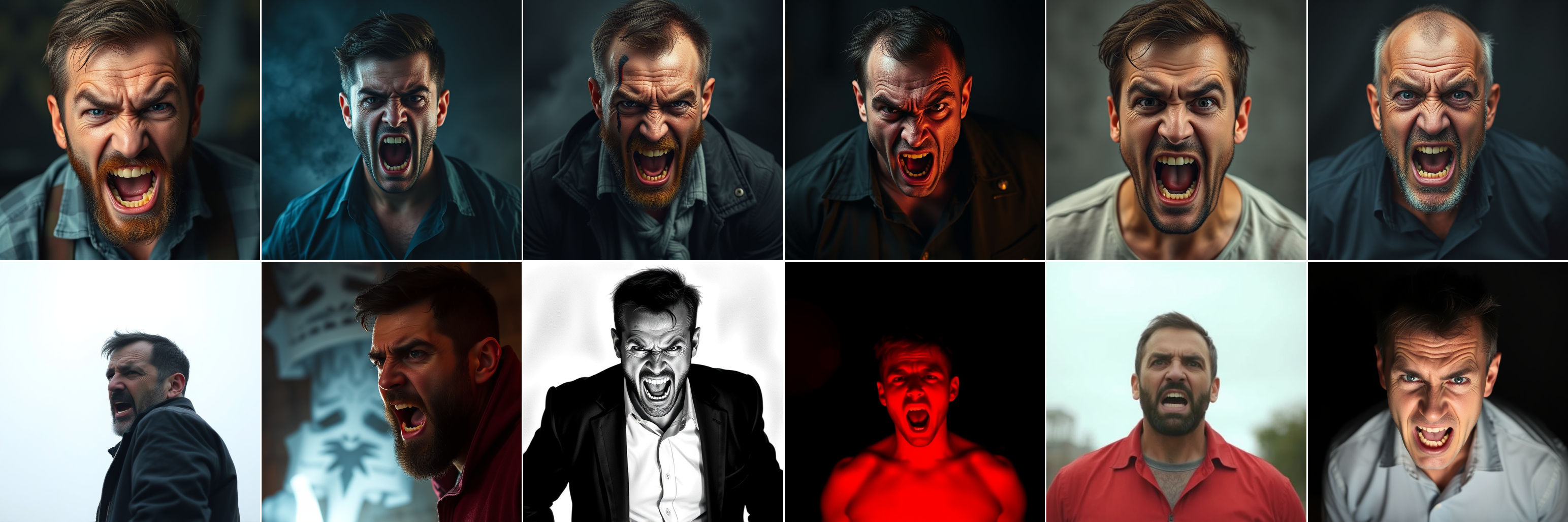}\\[1pt]
    {\small an angry man}
    \caption{\textbf{Qualitative diversity across twelve prompts (continued).} Prompts 9-12. Upper rows show unperturbed samples and lower rows show JIVE samples generated from the same initial seeds.}
\end{figure}

\clearpage

\subsection{Category-level results}
\label{app:category-results}

\begin{table}[H]
\caption{Category-level results on the curated prompt set. Each entry is the mean over 10 prompts and 128 images per arm. Panels (a) and (b) cover four categories each. Shading marks the primary Feature Vendi criteria. Boldface denotes the best result within each category and tier. For all metrics, higher values are better.}
\label{tab:per_category}
\centering
\footnotesize
\renewcommand{\arraystretch}{1.05}
\setlength{\tabcolsep}{2.5pt}

\begin{subtable}{\linewidth}
\caption{Objects, Human/Animals, Nature, and Scenes.}
\label{tab:per_category_1}
\resizebox{\linewidth}{!}{%
\begin{tabular}{@{}ll *{4}{@{\hspace{8pt}}ccc}@{}}
\toprule
& & \multicolumn{3}{c}{Objects}
& \multicolumn{3}{c}{Human/Animals}
& \multicolumn{3}{c}{Nature}
& \multicolumn{3}{c}{Scenes} \\
\cmidrule(lr){3-5}
\cmidrule(lr){6-8}
\cmidrule(lr){9-11}
\cmidrule(lr){12-14}
Tier & Metric
& Det. & OSCAR & JIVE
& Det. & OSCAR & JIVE
& Det. & OSCAR & JIVE
& Det. & OSCAR & JIVE \\
\midrule

\multirow{8}{*}{\textsc{General}}
& Feature Vendi
& 3.42
& 3.48
& \textbf{4.58}
& 5.10
& 5.17
& \textbf{7.41}
& 3.62
& 3.69
& \textbf{6.09}
& 4.25
& 4.36
& \textbf{5.59} \\

& Feature Vendi (grouped)
& 2.85
& 2.89
& \textbf{3.59}
& 3.76
& 3.79
& \textbf{5.11}
& 2.94
& 2.99
& \textbf{4.54}
& 3.36
& 3.43
& \textbf{4.16} \\

\addlinespace
& Pixel Vendi
& \textbf{10.35} & 10.29 & 10.29
& 10.75 & 10.68 & \textbf{10.98}
& 5.55 & 5.47 & \textbf{7.01}
& 11.42 & 11.37 & \textbf{12.10} \\

& Pixel Vendi (grouped)
& 6.86 & 6.83 & \textbf{7.15}
& 6.80 & 6.78 & \textbf{7.34}
& 4.11 & 4.06 & \textbf{5.45}
& 6.89 & 6.87 & \textbf{7.59} \\

\addlinespace
& Mean pairwise \(L_2\)
& 299.9 & 303.6 & \textbf{428.2}
& 286.4 & 289.1 & \textbf{412.7}
& 243.2 & 246.0 & \textbf{394.2}
& 290.3 & 294.8 & \textbf{367.7} \\

\addlinespace
& CLIP score
& 31.83 & \textbf{31.84} & 31.31
& 32.84 & \textbf{32.92} & 32.38
& 32.25 & \textbf{32.26} & 30.55
& 31.64 & \textbf{31.66} & 31.28 \\

& CLIP-IQA
& \textbf{0.503} & \textbf{0.503} & 0.502
& \textbf{0.503} & \textbf{0.503} & 0.502
& \textbf{0.503} & \textbf{0.503} & 0.501
& \textbf{0.502} & \textbf{0.502} & \textbf{0.502} \\

& HPSv2
& \textbf{0.298} & 0.297 & 0.273
& \textbf{0.319} & 0.318 & 0.292
& 0.287 & \textbf{0.288} & 0.249
& \textbf{0.312} & \textbf{0.312} & 0.290 \\

\midrule

\multirow{8}{*}{\textsc{Detailed}}
& Feature Vendi
& 2.23
& 2.23
& \textbf{2.73}
& 3.55
& 3.62
& \textbf{4.19}
& 2.38
& 2.39
& \textbf{3.37}
& 3.34
& 3.36
& \textbf{3.75} \\

& Feature Vendi (grouped)
& 1.99
& 1.99
& \textbf{2.36}
& 2.91
& 2.95
& \textbf{3.34}
& 2.10
& 2.10
& \textbf{2.79}
& 2.78
& 2.79
& \textbf{3.06} \\

\addlinespace
& Pixel Vendi
& 9.86 & 9.79 & \textbf{10.14}
& 8.42 & 8.32 & \textbf{9.51}
& 4.86 & 4.75 & \textbf{6.60}
& 10.24 & 10.20 & \textbf{10.90} \\

& Pixel Vendi (grouped)
& 6.52 & 6.49 & \textbf{6.97}
& 5.58 & 5.53 & \textbf{6.46}
& 3.69 & 3.62 & \textbf{5.06}
& 6.32 & 6.30 & \textbf{6.86} \\

\addlinespace
& Mean pairwise \(L_2\)
& 275.3 & 278.6 & \textbf{388.2}
& 255.2 & 256.9 & \textbf{351.3}
& 228.9 & 231.0 & \textbf{353.8}
& 270.5 & 274.4 & \textbf{321.4} \\

\addlinespace
& CLIP score
& 33.23 & \textbf{33.28} & 32.93
& 34.63 & \textbf{34.65} & 34.36
& 31.16 & \textbf{31.19} & 30.47
& \textbf{32.21} & 32.20 & 31.89 \\

& CLIP-IQA
& \textbf{0.503} & \textbf{0.503} & \textbf{0.503}
& \textbf{0.503} & \textbf{0.503} & 0.502
& \textbf{0.503} & \textbf{0.503} & 0.502
& \textbf{0.502} & \textbf{0.502} & \textbf{0.502} \\

& HPSv2
& \textbf{0.319} & 0.318 & 0.303
& \textbf{0.321} & 0.320 & 0.308
& 0.278 & \textbf{0.279} & 0.253
& \textbf{0.311} & \textbf{0.311} & 0.298 \\

\bottomrule
\end{tabular}%
}
\end{subtable}

\vspace{0.1cm}

\begin{subtable}{\linewidth}
\caption{Characters, Food, Art, and Text.}
\label{tab:per_category_2}
\resizebox{\linewidth}{!}{%
\begin{tabular}{@{}ll *{4}{@{\hspace{8pt}}ccc}@{}}
\toprule
& & \multicolumn{3}{c}{Characters}
& \multicolumn{3}{c}{Food}
& \multicolumn{3}{c}{Art}
& \multicolumn{3}{c}{Text} \\
\cmidrule(lr){3-5}
\cmidrule(lr){6-8}
\cmidrule(lr){9-11}
\cmidrule(lr){12-14}
Tier & Metric
& Det. & OSCAR & JIVE
& Det. & OSCAR & JIVE
& Det. & OSCAR & JIVE
& Det. & OSCAR & JIVE \\
\midrule

\multirow{8}{*}{\textsc{General}}
& Feature Vendi
& 3.24
& 3.23
& \textbf{4.22}
& 2.27
& 2.25
& \textbf{3.07}
& 4.21
& 4.18
& \textbf{5.37}
& 4.22
& 4.42
& \textbf{7.35} \\

& Feature Vendi (grouped)
& 2.70
& 2.70
& \textbf{3.38}
& 2.03
& 2.02
& \textbf{2.62}
& 3.34
& 3.30
& \textbf{4.06}
& 3.26
& 3.36
& \textbf{5.12} \\

\addlinespace
& Pixel Vendi
& 9.97 & 9.91 & \textbf{10.75}
& 9.27 & 9.20 & \textbf{10.29}
& 10.00 & 10.00 & \textbf{10.92}
& \textbf{11.09} & 11.04 & 10.76 \\

& Pixel Vendi (grouped)
& 6.32 & 6.29 & \textbf{7.04}
& 6.17 & 6.14 & \textbf{6.99}
& 6.24 & 6.24 & \textbf{7.04}
& 6.90 & 6.88 & \textbf{7.22} \\

\addlinespace
& Mean pairwise \(L_2\)
& 299.9 & 303.4 & \textbf{413.6}
& 326.4 & 329.9 & \textbf{434.8}
& 306.1 & 308.3 & \textbf{401.9}
& 336.1 & 340.5 & \textbf{441.0} \\

\addlinespace
& CLIP score
& 35.49 & \textbf{35.55} & 34.81
& 31.30 & \textbf{31.33} & 31.25
& 33.25 & \textbf{33.29} & 32.16
& 32.76 & \textbf{32.79} & 32.63 \\

& CLIP-IQA
& \textbf{0.504} & \textbf{0.504} & 0.503
& \textbf{0.503} & \textbf{0.503} & 0.502
& \textbf{0.503} & \textbf{0.503} & \textbf{0.503}
& \textbf{0.503} & \textbf{0.503} & 0.502 \\

& HPSv2
& \textbf{0.338} & \textbf{0.338} & 0.316
& \textbf{0.303} & \textbf{0.303} & 0.281
& 0.289 & \textbf{0.290} & 0.269
& 0.294 & \textbf{0.295} & 0.266 \\

\midrule

\multirow{8}{*}{\textsc{Detailed}}
& Feature Vendi
& 3.23
& 3.27
& \textbf{3.67}
& 2.06
& 2.06
& \textbf{2.26}
& 3.26
& 3.21
& \textbf{3.66}
& 3.13
& 3.22
& \textbf{3.91} \\

& Feature Vendi (grouped)
& 2.68
& 2.71
& \textbf{3.00}
& 1.86
& 1.86
& \textbf{2.02}
& 2.70
& 2.66
& \textbf{2.98}
& 2.59
& 2.65
& \textbf{3.10} \\

\addlinespace
& Pixel Vendi
& 9.46 & 9.40 & \textbf{10.41}
& 8.50 & 8.44 & \textbf{9.68}
& 8.98 & 8.95 & \textbf{9.99}
& 10.15 & 10.09 & \textbf{10.86} \\

& Pixel Vendi (grouped)
& 5.99 & 5.96 & \textbf{6.71}
& 5.63 & 5.61 & \textbf{6.53}
& 5.65 & 5.64 & \textbf{6.36}
& 6.40 & 6.37 & \textbf{7.05} \\

\addlinespace
& Mean pairwise \(L_2\)
& 281.4 & 285.1 & \textbf{349.2}
& 303.2 & 306.4 & \textbf{385.8}
& 287.2 & 288.9 & \textbf{344.8}
& 290.0 & 293.6 & \textbf{383.9} \\

\addlinespace
& CLIP score
& 35.09 & \textbf{35.20} & 34.79
& \textbf{31.24} & 31.21 & 31.15
& 33.21 & \textbf{33.25} & 32.86
& 32.57 & \textbf{32.60} & 32.23 \\

& CLIP-IQA
& \textbf{0.504} & \textbf{0.504} & 0.503
& \textbf{0.503} & \textbf{0.503} & 0.502
& \textbf{0.503} & \textbf{0.503} & 0.502
& \textbf{0.503} & \textbf{0.503} & 0.502 \\

& HPSv2
& \textbf{0.329} & \textbf{0.329} & 0.319
& \textbf{0.300} & 0.299 & 0.286
& 0.287 & \textbf{0.288} & 0.273
& \textbf{0.289} & \textbf{0.289} & 0.276 \\

\bottomrule
\end{tabular}%
}
\end{subtable}
\end{table}

\subsection{Detailed prompts}
\label{app:detailed-prompts}
Each entry below is index-matched, within its category, to the correspondingly
numbered entry in Appendix~\ref{app:general-prompts}: the two describe the same
underlying scene and differ only in specificity. This pairing is what makes the
tier contrast in Appendix~\ref{app:category-results} a controlled comparison.

\subsubsection{O --- Object}
\begin{enumerate}[label=O\arabic*., leftmargin=*]
  \item \textit{A glossy red sports car parked on a clean mountain road, viewed from a low front angle in soft morning light.}
  \item \textit{A handmade blue ceramic teapot with a curved handle on a rustic wooden table beside a folded linen cloth.}
  \item \textit{A black vintage film camera beside an open photography book on a dark wooden desk near a window.}
  \item \textit{A vintage brass compass resting on an old nautical map, lit by a flickering candle in a dim study.}
  \item \textit{A stack of three worn leather-bound books with gold lettering, tied with a frayed ribbon, on a dark mahogany shelf.}
  \item \textit{A clear glass vase holding a single white orchid, placed on a marble countertop near a sunlit window.}
  \item \textit{A luxury silver wristwatch with a deep blue dial and steel bracelet, resting on a polished black stone surface under dramatic studio lighting.}
  \item \textit{A well-worn acoustic guitar with a sunburst finish leaning against a vintage leather armchair in a warmly lit music room.}
  \item \textit{A pristine white leather sneaker displayed on a minimalist concrete pedestal in a brightly lit retail showroom.}
  \item \textit{A round red paper lantern glowing warmly, hanging from a dark wooden beam against a deep blue twilight sky.}
\end{enumerate}

\subsubsection{H --- Human/animal}
\begin{enumerate}[label=H\arabic*., leftmargin=*]
  \item \textit{A golden retriever sitting on green grass in a city park, looking toward the camera in warm afternoon sunlight.}
  \item \textit{A young woman wearing a long black coat walking along a busy city sidewalk on an overcast winter afternoon.}
  \item \textit{A small blue bird standing on a thin flowering tree branch against a softly blurred forest background.}
  \item \textit{A young girl in a bright yellow raincoat jumping over a puddle, splashing water on a rainy autumn street.}
  \item \textit{A black cat curled up asleep on a patterned rug in front of a crackling fireplace, with a steaming mug nearby.}
  \item \textit{An elderly man feeding pigeons on a wooden park bench, surrounded by golden autumn leaves in soft morning light.}
  \item \textit{A ballerina in a pale pink tutu stretching at a wooden barre in a sunlit dance studio with tall mirrors.}
  \item \textit{A red fox with a thick winter coat walking through fresh snow in a quiet birch forest, its breath visible in the cold air.}
  \item \textit{A young boy in a striped shirt flying a red diamond kite on a windy grassy hill under a bright blue sky with scattered clouds.}
  \item \textit{A barn owl with speckled feathers perched on a weathered wooden fence post at dusk, with a misty meadow fading into the background.}
\end{enumerate}

\subsubsection{N --- Natural}
\begin{enumerate}[label=N\arabic*., leftmargin=*]
  \item \textit{A tropical beach with several palm trees, turquoise water, white sand, and scattered clouds at midday.}
  \item \textit{A calm alpine lake surrounded by tall pine trees and snow-covered mountains, with reflections visible in the water.}
  \item \textit{A tall waterfall flowing into a clear pool inside a dense tropical forest covered with moss and broad green leaves.}
  \item \textit{A rolling green meadow dotted with wildflowers, a single large oak tree, and a distant mountain range under a vivid blue sky.}
  \item \textit{A frozen lake under a starry night sky, with snow-covered evergreens lining the shore and the aurora borealis glowing overhead.}
  \item \textit{A dramatic thunderstorm over a vast desert canyon, lightning striking the layered rock formations as dark clouds swirl above.}
  \item \textit{A dense bamboo forest with tall green stalks, narrow sunbeams cutting through the canopy onto a stone path covered in fallen leaves.}
  \item \textit{Sweeping red sand dunes rippled by wind, glowing orange under a low desert sun with long shadows and a cloudless gradient sky.}
  \item \textit{A slow river winding through a misty valley at dawn, with soft pink light on the fog and dark silhouettes of willow trees along the banks.}
  \item \textit{A rugged volcanic coastline with black sand beaches, crashing white waves, and jagged basalt cliffs under a dramatic overcast sky.}
\end{enumerate}

\subsubsection{S --- Structured}
\begin{enumerate}[label=S\arabic*., leftmargin=*]
  \item \textit{A minimalist modern living room with floor-to-ceiling windows, a gray sofa, a wooden coffee table, and bright natural light.}
  \item \textit{A crowded city street at night with illuminated storefronts, moving cars, wet pavement, and colorful reflections.}
  \item \textit{A historic European town square with a stone fountain in the center, outdoor caf\'{e}s, and pastel-colored buildings.}
  \item \textit{A steampunk airship docked at a floating Victorian-era platform, with brass gears, riveted panels, and billowing steam in a sunset sky.}
  \item \textit{A futuristic neon-lit alley in a cyberpunk city, with holographic signs, steam vents, and a lone figure in a dark trench coat.}
  \item \textit{A medieval castle interior with a long wooden banquet table set for a feast, torches flickering on stone walls, and colorful stained glass windows.}
  \item \textit{A cozy log cabin with warm golden windows and smoke curling from its stone chimney, nestled among snow-laden pine trees at blue hour.}
  \item \textit{A busy underground subway station during rush hour, with blurred commuters, an arriving train's headlights, and tiled walls lined with advertisements.}
  \item \textit{A serene Japanese temple garden with a five-story red pagoda, a koi pond with a wooden arched bridge, and carefully raked gravel under cherry blossoms.}
  \item \textit{A grand old library with towering wooden bookshelves, wrought-iron spiral staircases, green reading lamps, and dust motes floating in shafts of window light.}
\end{enumerate}

\subsubsection{C --- Compositional}
\begin{enumerate}[label=C\arabic*., leftmargin=*]
  \item \textit{A chef slicing vegetables on a wooden cutting board in a professional kitchen, with cooking utensils hanging behind them.}
  \item \textit{A panda wearing a white spacesuit standing on the moon, with a small flag nearby and Earth visible in the background.}
  \item \textit{A blue bicycle leaning against a red brick wall, with a brown basket attached to the front and fallen leaves on the ground.}
  \item \textit{A giant octopus wearing a top hat and monocle, playing a grand piano in an underwater ballroom with glowing jellyfish chandeliers.}
  \item \textit{A tiny astronaut exploring a giant sunflower, with dewdrops like planets and a ladybug resting nearby like a parked spaceship.}
  \item \textit{A steampunk owl with copper feathers and brass goggles, perched on a mechanical branch in a clockwork forest.}
  \item \textit{A giraffe wearing a knitted red scarf riding a vintage mint-green scooter down a cobblestone street, with market stalls blurring past.}
  \item \textit{A small round robot with brass fittings gently watering potted tomato plants on a sunny apartment balcony overlooking a quiet city.}
  \item \textit{A giant blue whale floating serenely above a golden desert at dusk, casting a long shadow over the dunes as birds fly alongside it.}
  \item \textit{A medieval knight in polished armor playing chess with a small emerald dragon at a stone table inside a torch-lit castle tower.}
\end{enumerate}

\subsubsection{F --- Food}
\begin{enumerate}[label=F\arabic*., leftmargin=*]
  \item \textit{A gourmet cheeseburger with melted cheddar, caramelized onions, and a brioche bun on a rustic ceramic plate beside golden fries.}
  \item \textit{A steaming bowl of ramen with a soft-boiled egg, chashu pork, and scallions, with wooden chopsticks resting on the rim in a cozy shop.}
  \item \textit{A tall stack of fluffy pancakes with melting butter and maple syrup dripping down the sides, topped with fresh blueberries on a white plate.}
  \item \textit{A ceramic cup of latte with intricate rosetta foam art on a wooden caf\'{e} table beside a small vase of dried flowers.}
  \item \textit{A three-layer dark chocolate cake with glossy ganache, topped with fresh raspberries and mint leaves on an elegant cake stand.}
  \item \textit{A glossy fruit tart with kiwi, strawberries, and blueberries arranged in a spiral, displayed on a marble bakery shelf under warm light.}
  \item \textit{An artfully arranged sushi platter with nigiri, maki rolls, pickled ginger, and wasabi on a dark wooden board with chopsticks and soy sauce.}
  \item \textit{A wood-fired margherita pizza with bubbling mozzarella and fresh basil being sliced on a wooden peel in a rustic Italian kitchen.}
  \item \textit{A waffle cone with two scoops of pastel pink strawberry and mint ice cream, beginning to melt, against a bright summer boardwalk backdrop.}
  \item \textit{A morning breakfast table with a French press of coffee, flaky croissants, orange juice, and a folded newspaper in soft window light.}
\end{enumerate}

\subsubsection{A --- Artistic}
\begin{enumerate}[label=A\arabic*., leftmargin=*]
  \item \textit{A loose watercolor painting of a quiet harbor at dawn, with soft washes of pink and gray, moored sailboats, and gentle reflections bleeding into the paper.}
  \item \textit{An impressionist oil painting of a sunflower field under a swirling turquoise sky, with thick visible brushstrokes and rich impasto texture.}
  \item \textit{A bold geometric abstract composition of interlocking circles and triangles in terracotta, teal, and cream, with subtle grain texture and balanced asymmetry.}
  \item \textit{An expressive charcoal sketch of a dancer mid-leap, with sweeping gestural lines, smudged shadows, and dynamic movement captured on textured paper.}
  \item \textit{A layered paper-cut artwork of a forest with deer, crafted from sheets of green and gold paper, casting delicate shadows in a lightbox display.}
  \item \textit{A vibrant stained glass window depicting a rising phoenix, with ruby and amber glass segments outlined in dark lead, glowing in transmitted light.}
  \item \textit{An intricate ceramic tile mosaic of a sea turtle swimming through coral, made of hundreds of small turquoise, jade, and sand-colored tiles.}
  \item \textit{A surreal pen-and-ink drawing of a floating city with impossible staircases and hanging gardens, rendered in fine crosshatching on aged paper.}
  \item \textit{A low-poly 3D render of a snow-capped mountain and pine forest at sunrise, with flat-shaded facets in soft gradients of pink, purple, and blue.}
  \item \textit{A vibrant graffiti mural of a roaring lion with a mane of swirling colorful patterns, spray-painted on a brick wall with paint drips.}
\end{enumerate}

\subsubsection{T --- Text/signage}
\begin{enumerate}[label=T\arabic*., leftmargin=*]
  \item \textit{A buzzing red neon sign spelling `Open' in cursive, mounted on the dark window of a late-night diner with rain outside.}
  \item \textit{A hand-carved wooden sign reading `Welcome' in rustic lettering, hanging from chains on the door of a countryside bed and breakfast.}
  \item \textit{A retro science-fiction movie poster titled `Starlight' in bold chrome lettering, showing a rocket soaring past a ringed planet.}
  \item \textit{A cozy caf\'{e} chalkboard menu with `Today's Special: Pumpkin Soup' written in decorative white chalk lettering, framed by drawn vines.}
  \item \textit{A vintage art-deco travel poster of Paris with `Visit Paris' in elegant typography above a stylized Eiffel Tower at sunset.}
  \item \textit{A moody book cover titled `The Last Lighthouse' in embossed serif type, showing a lone lighthouse on a cliff under storm clouds.}
  \item \textit{A large street mural with the word `Hope' painted in rainbow-gradient block letters across a city wall, with small birds perched on the letters.}
  \item \textit{A pastel birthday cake with `Happy Birthday' piped in blue icing, surrounded by lit candles and colorful sprinkles on a party table.}
  \item \textit{A split-flap airport departure board displaying `Tokyo 09:45 On Time' in white letters, glowing in a busy terminal.}
  \item \textit{A classic theater marquee with `Tonight' in glowing bulb letters, above a red velvet entrance on a rainy evening street.}
\end{enumerate}

\clearpage
\subsection{General prompts}
\label{app:general-prompts}
The general-tier counterparts of Appendix~\ref{app:detailed-prompts}, matched by
category and index.

\begin{multicols}{2}
\setlist[enumerate]{itemsep=1pt,topsep=3pt,parsep=0pt}

\subsubsection{O --- Object}
\begin{enumerate}[label=O\arabic*., leftmargin=*]
  \item \textit{A red sports car on a road.}
  \item \textit{A ceramic teapot on a table.}
  \item \textit{A vintage camera beside a book.}
  \item \textit{A brass compass on a map.}
  \item \textit{A stack of old books on a shelf.}
  \item \textit{A glass vase with a flower on a counter.}
  \item \textit{A silver wristwatch on a black surface.}
  \item \textit{An acoustic guitar leaning against a chair.}
  \item \textit{A white sneaker on a pedestal.}
  \item \textit{A paper lantern glowing in the dark.}
\end{enumerate}

\subsubsection{H --- Human/animal}
\begin{enumerate}[label=H\arabic*., leftmargin=*]
  \item \textit{A golden retriever in a park.}
  \item \textit{A woman wearing a black coat in the city.}
  \item \textit{A bird standing on a tree branch.}
  \item \textit{A girl in a raincoat jumping over a puddle.}
  \item \textit{A cat sleeping on a rug by a fireplace.}
  \item \textit{A man feeding pigeons on a park bench.}
  \item \textit{A ballerina stretching at a barre.}
  \item \textit{A fox walking through snow.}
  \item \textit{A boy flying a kite on a hill.}
  \item \textit{An owl perched on a fence post.}
\end{enumerate}

\subsubsection{N --- Natural}
\begin{enumerate}[label=N\arabic*., leftmargin=*]
  \item \textit{A beach with palm trees.}
  \item \textit{A mountain lake surrounded by pine trees.}
  \item \textit{A waterfall in a tropical forest.}
  \item \textit{A meadow with an oak tree and mountains.}
  \item \textit{A frozen lake at night with trees.}
  \item \textit{A thunderstorm over a canyon.}
  \item \textit{A bamboo forest with sunbeams.}
  \item \textit{A desert dune at sunset.}
  \item \textit{A misty river at dawn.}
  \item \textit{A volcanic coastline with black sand.}
\end{enumerate}

\subsubsection{S --- Structured}
\begin{enumerate}[label=S\arabic*., leftmargin=*]
  \item \textit{A modern living room with large windows.}
  \item \textit{A busy city street at night.}
  \item \textit{A European town square with a fountain.}
  \item \textit{A steampunk airship at a platform.}
  \item \textit{A futuristic neon alley in a city.}
  \item \textit{A medieval castle banquet hall.}
  \item \textit{A cozy mountain cabin in winter.}
  \item \textit{A subway station during rush hour.}
  \item \textit{A Japanese temple garden with a pagoda.}
  \item \textit{An old library with spiral staircases.}
\end{enumerate}

\subsubsection{C --- Compositional}
\begin{enumerate}[label=C\arabic*., leftmargin=*]
  \item \textit{A chef preparing food in a kitchen.}
  \item \textit{A panda wearing a spacesuit on the moon.}
  \item \textit{A blue bicycle against a brick wall.}
  \item \textit{An octopus playing a piano underwater.}
  \item \textit{A tiny astronaut on a sunflower.}
  \item \textit{A steampunk owl on a branch.}
  \item \textit{A giraffe wearing a scarf riding a scooter.}
  \item \textit{A robot watering plants on a balcony.}
  \item \textit{A whale floating above a desert.}
  \item \textit{A knight playing chess with a dragon.}
\end{enumerate}

\subsubsection{F --- Food}
\begin{enumerate}[label=F\arabic*., leftmargin=*]
  \item \textit{A cheeseburger on a plate.}
  \item \textit{A bowl of ramen with chopsticks.}
  \item \textit{A stack of pancakes with syrup.}
  \item \textit{A cup of latte with foam art.}
  \item \textit{A chocolate cake with berries.}
  \item \textit{A fruit tart on a bakery shelf.}
  \item \textit{A sushi platter on a wooden board.}
  \item \textit{A pizza being sliced.}
  \item \textit{An ice cream cone with two scoops.}
  \item \textit{A breakfast table with coffee and croissants.}
\end{enumerate}

\subsubsection{A --- Artistic}
\begin{enumerate}[label=A\arabic*., leftmargin=*]
  \item \textit{A watercolor painting of a harbor.}
  \item \textit{An oil painting of a sunflower field.}
  \item \textit{A geometric abstract composition.}
  \item \textit{A charcoal sketch of a dancer.}
  \item \textit{A paper-cut artwork of a forest.}
  \item \textit{A stained glass window of a phoenix.}
  \item \textit{A mosaic of a sea turtle.}
  \item \textit{A surreal ink drawing of a floating city.}
  \item \textit{A low-poly 3D render of a mountain.}
  \item \textit{A graffiti mural of a lion.}
\end{enumerate}

\subsubsection{T --- Text/signage}
\begin{enumerate}[label=T\arabic*., leftmargin=*]
  \item \textit{A neon sign that says `Open'.}
  \item \textit{A wooden sign that says `Welcome'.}
  \item \textit{A movie poster titled `Starlight'.}
  \item \textit{A caf\'{e} chalkboard menu.}
  \item \textit{A vintage travel poster of Paris.}
  \item \textit{A book cover titled `The Last Lighthouse'.}
  \item \textit{A street mural with the word `Hope'.}
  \item \textit{A birthday cake that says `Happy Birthday'.}
  \item \textit{An airport departure board.}
  \item \textit{A marquee sign that says `Tonight'.}
\end{enumerate}

\end{multicols}

\end{document}